\documentclass[journal]{IEEEtran}
\usepackage[T1]{fontenc}
\usepackage[utf8]{inputenc}
\usepackage{amsmath,amssymb,amsfonts,amsthm}
\usepackage{algorithmic}
\usepackage{algorithm}
\usepackage{array}
\usepackage[caption=false,font=normalsize,labelfont=sf,textfont=sf]{subfig}
\usepackage{mathtools}
\usepackage{booktabs}
\usepackage{multirow}
\usepackage{makecell}
\usepackage{colortbl}
\usepackage{tabularx}
\usepackage{tikz}
\usepackage{xcolor}
\usepackage{pifont}
\usepackage{enumitem}
\usepackage{threeparttable}
\usepackage{textcomp}
\usepackage{stfloats}
\usepackage{url}
\usepackage{verbatim}
\usepackage{graphicx}
\usepackage{hyperref}
\usepackage{cite}
\definecolor{darkgreen}{HTML}{00C700}

\definecolor{algBlue}{RGB}{210,225,255}   
\definecolor{algRed}{RGB}{255,210,210}

\usetikzlibrary{arrows.meta,positioning,fit,calc}

\newtheorem{theorem}{Theorem}
\newtheorem{lemma}{Lemma}
\newtheorem{proposition}{Proposition}
\newtheorem{definition}{Definition}
\newtheorem{assumption}{Assumption}
\newtheorem{remark}{Remark}

\newcommand{\cmark}{\ding{51}} 
\newcommand{\xmark}{\ding{55}} 

\newcommand{\hlblue}[1]{\colorbox{algBlue}{\strut #1}}
\newcommand{\hlred}[1]{\colorbox{algRed}{\strut #1}}

\begin{document}

\title{Asynchronous Decentralized Federated Learning over Lossy Wireless Links via Reception- and Age-Aware Aggregation}

\author{\IEEEauthorblockN{Chanuka AS Hewa Kaluannakkage,~\IEEEmembership{Member,~IEEE and Rajkumar Buyya,~\IEEEmembership{Fellow,~IEEE}\\
\IEEEauthorblockA{\textit{Quantum Cloud and Distributed Systems (qCLOUDS) Lab} \\
\textit{School of Computing and Information Systems}\\
The University of Melbourne, Australia \\
Email: \{c.hewakaluannakkage, rbuyya\}@unimelb.edu.au}
}}}



\maketitle

\begin{abstract}
Decentralized Federated Learning(DFL) enables collaborative model training across wireless edge nodes, including IoT deployments, autonomous vehicles, UAV swarms, and satellite constellations. Operating over lossy wireless links under constraints, these systems cannot rely on retransmissions, so model parameters must be accepted as partial chunks, leading to two key failure modes, which are selection bias, where poor-quality links are systematically under-represented in gossip aggregation, and update staleness, where asynchronous nodes contribute outdated models. We prove that classical gossip aggregation introduces irreducible selection bias proportional to the link-loss rate. We propose DFL-AA (Decentralized Federated Learning with Adaptive AoI-weighted Aggregation), which corrects selection bias using Inverse Probability Weighting (IPW) with online channel estimation and mitigates staleness via Age-of-Information (AoI) decay without requiring a global clock. We prove that DFL-AA removes link-quality distortion in expectation and consistently outperforms state-of-the-art baselines across varying loss rates and heterogeneous channel conditions on fixed directed topologies.

\end{abstract}


\section{Introduction}
\label{sec:introduction}

\IEEEPARstart{D}{istributed} edge intelligence in wireless sensor networks and IoT environments increasingly relies on Decentralized Federated Learning (DFL) \cite{dflsurver2024, CFLvsDFL} rather than centralized alternatives \cite{mammen2021federated, FedMD, FedProx}, enabling collaborative model training without requiring the centralization of raw data. DFL is particularly suitable for applications such as industrial IoT systems, smart city infrastructures, autonomous vehicle networks, and UAV swarms, where data is generated in a distributed manner at the network edge. However, these environments typically operate over unreliable wireless communication channels and are subject to strict latency constraints, making retransmission-based communication mechanisms impractical. As a result, maintaining efficient and robust decentralized learning under lossy communication conditions remains a significant challenge.

(1) \textit{Unreliable Communication:} Wireless communication links widely used in networks experience high packet loss due to distance, interference, and obstacles. Unlike data centers with $< 1\%$ loss \cite{ATP}, wireless networks cannot rely on TCP (Transmission Control Protocol) retransmissions without violating the low-latency requirements of real-time decision-making. Additionally, industrial evidence also verified that retransmission overhead caused by TCP-like protocols is a significant performance bottleneck in production deployments \cite{NSDI_DCTCP_24, FedScale_Google, microsoft_tech_report}. Figure~\ref{fig:packet_reception} illustrates the expected data loss over an unreliable communication channel for a given ML model. In this regime, nodes must aggregate whatever partial model information is available at each communication round rather than waiting for complete reception. In practice, per-link channel quality varies across the network due to differences in distance and interference, motivating per-link channel estimation rather than a single global loss parameter.

\begin{figure}[!t]
\centering
\includegraphics[width=0.95\columnwidth]{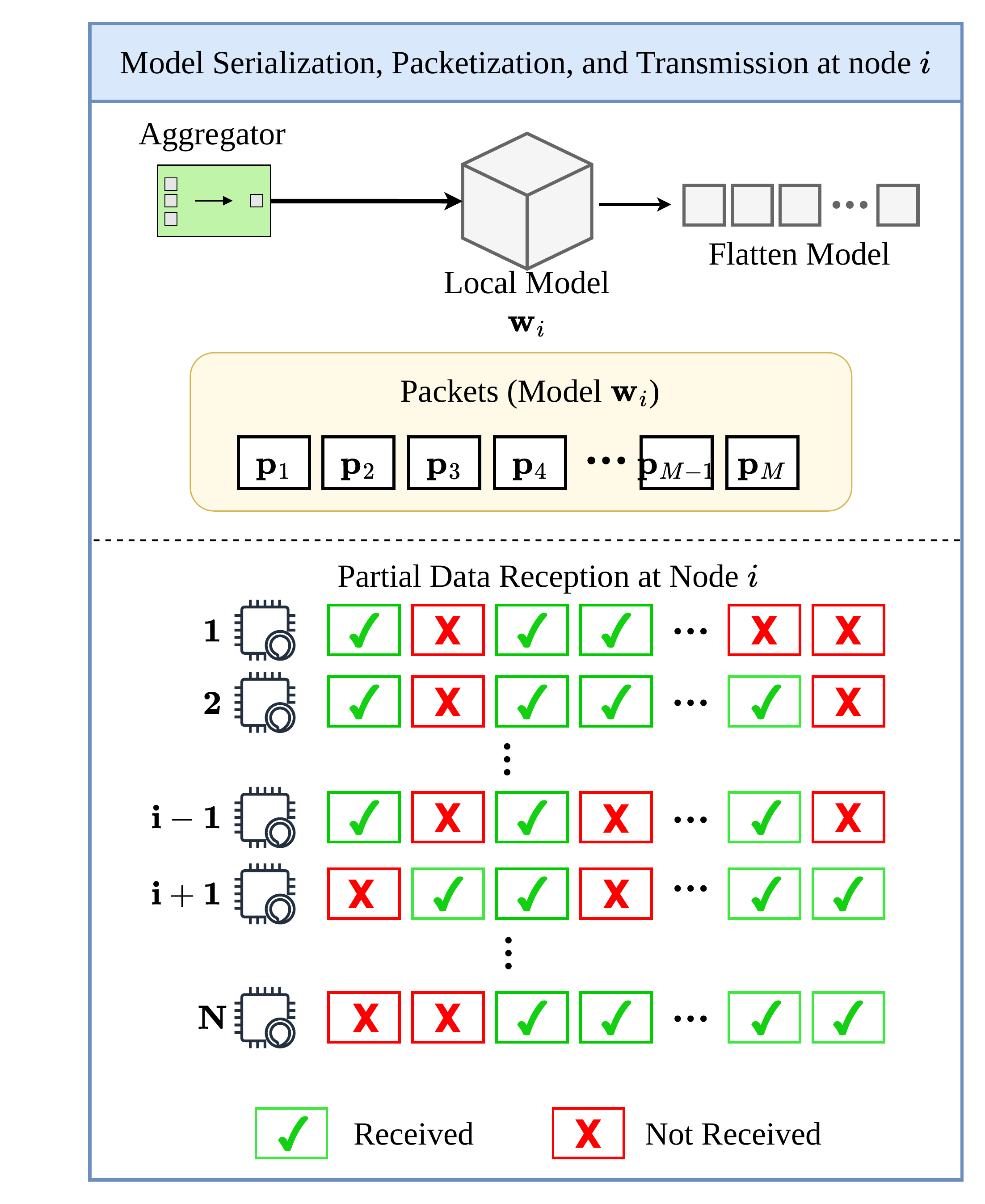}
\caption{Model serialization, packetization, and partial reception over unreliable communication in a peer-to-peer system where node $i$ receives from its neighbors.}
\label{fig:packet_reception}
\end{figure}

(2) \textit{Heterogeneity and Asynchrony:} System heterogeneity in computation speed and link quality, combined with non-IID (Independent and Identically Distributed) data distributions across nodes, causes gradient divergence. In wireless deployments, nodes train and communicate at different rates, determined by local compute capacity and link availability, resulting in genuinely asynchronous behavior with no global synchronization barrier. Figure \ref{fig:async_behavior} further illustrates the asynchronous behavior of decentralized training and communication of peer nodes due to their heterogeneous capabilities. Partial model reception under these conditions creates asymmetric information availability, in which nodes with poor links receive fewer of their neighbors' models over rounds, amplifying the effect of data heterogeneity on aggregation quality.

\begin{figure}[!t]
    \centering
    \includegraphics[width=0.95\columnwidth]{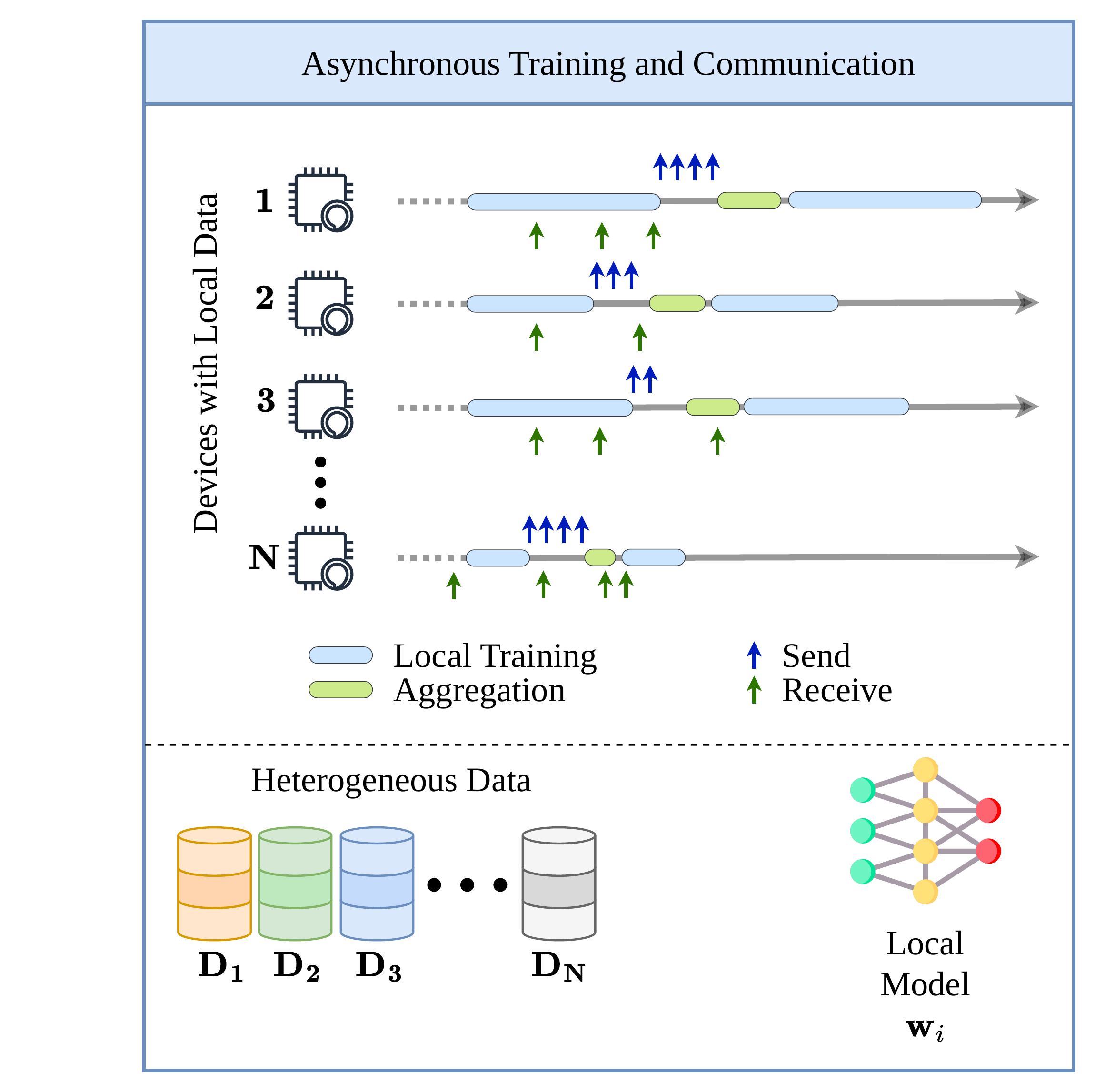}
    \caption{Asynchronous behavior of decentralized systems due to heterogeneous capabilities. Each device completes its training, receives from others, and broadcasts at a given time, but it's not synced with others.}
    \label{fig:async_behavior}
\end{figure}

\subsection{The Problem: Partial and Stale Updates} 
\label{sec:challenge}

When focusing on the low-latency \emph{fire-and-forget} communication regime, where retransmissions are infeasible, decentralized wireless networks face spatial incompleteness and temporal freshness issues.  Model parameters are serialized into fixed-size chunks. Each chunk is dropped independently, so the receiving node receives an incomplete parameter vector. Critically, these two failure modes are \emph{independent} in origin. Selection bias from spatial incompleteness arises from channel quality, while staleness arises from compute and communication heterogeneity and must be addressed jointly within a single aggregation rule to avoid compounding their effects.

\subsection{Motivation}
\label{sec:motivation}

Most existing production-level Federated Learning approaches assume reliable communication, even when they do not explicitly state so \cite{FedAvg, flower, FedScale_Google}. When it comes to DFL \cite{HADFL, adaptive_configurationfor_heterogeneous_participants_in_DFL, AsyncFL_in_decentralized_topology_on_dynamic_avg_consensus, dynamic-topology-optimization-for-efficient-and-DFL, DFL_het_edge, pFedgame_2024} and wireless DFL deployments focusing on vehicular systems \cite{chen2025mobility_aware_multi_tasking, mob-fl, drl_based_FL_for_efficient_vehicular_caching_mgmnt, a_FL_approach_to_QoS_forcasting_in_cellular_vehicular_comm}, most works still assume perfect communication. However, a lightweight solution that meets low-latency constraints is required to address the aforementioned challenge in DFL configurations, especially when operating in a wireless regime with no or minimal retransmissions. Systematic solutions in the literature for handling partial receptions without retransmissions either drop partially received updates or naively fill in the missing data from the local model \cite{DFL_with_unreliable_comm, partial_updates_veh} or zeros or maintain a complex network coding to reconstruct missing data chunks \cite{LRA_SysNC}. These approaches lead either to wasted communication resources by ignoring updates, to selection bias due to naive partial acceptance, or to extra overhead in heavy computation. Reconstruction methods for lossy channels alone cannot withstand realistic conditions in a wireless DFL setting, as described earlier. Moreover, the integration of model freshness effects into partial model reception in configurations remains limited, even though it should be properly managed given the inherent asynchrony. This gap is particularly acute in IoT and mobile network deployments, where per-link channel quality is determined by physical distance and interference, yet existing gossip aggregation protocols are limited in their ability to estimate and correct per-link reception rates online.

\subsection{Our Approach}
\label{sec:our_approach}

We observe that selection bias is a statistical sampling problem where the inbox is a biased sample of the neighborhood because links with lower reception rates appear less frequently. The classical remedy is the Horvitz-Thompson estimator~\cite{IPW}, which reweights each observed sample by the inverse of its inclusion probability. DFL-AA instantiates this principle in the gossip aggregation step with local-fill reconstruction for missing chunks: each model update(reconstructed) from neighbor $j$ is weighted by $1/\hat{q}_{ij}$, where $\hat{q}_{ij}$ is an online EWMA estimate of the observed link reception rate, up-weighting under-represented neighbors to restore a statistically unbiased sample of the neighborhood. This is combined with an AoI (Age of Information) \cite{AOI}-based exponential decay $\exp(-\mathrm{AoI}_{ij}/\tau)$, computed from generation timestamps in received messages without a global clock, to discount temporally stale contributions. The combined weight addresses both failure modes independently and multiplicatively, where a high weight requires both good channel quality and temporal freshness.

\subsection{Contributions}
\label{sec:contribution}

\begin{enumerate}[leftmargin=*, label=\textbf{C\arabic*}, noitemsep]

\item \textbf{Problem positioning.} We provide the first systematic study of the joint effect of selection bias and update staleness in asynchronous DFL over fixed directed graphs in wireless networks, revealing drastic performance degradation of state-of-the-art methods.

\item \textbf{Failure mode analysis.} We prove that uniform gossip under local-fill reconstruction distorts each neighbor's contribution by $(1-q_{ij})$~(Proposition~\ref{prop:lf_bias}), and that push-sum weight variables drain geometrically under chunk-level loss~(Proposition~\ref{prop:drain}), which is a fundamental study on directed graphs.

\item \textbf{DFL-AA algorithm.} We introduce DFL-AA (\textbf{D}ecentralized \textbf{F}ederated \textbf{L}earning with \textbf{A}daptive \textbf{A}oI-weighted Aggregation), an IPW-corrected, AoI-aware gossip aggregation rule for asynchronous fixed directed graphs, with online EWMA channel estimation requiring no coordination, no global clock, and no additional communication overhead.

\item \textbf{Unbiasedness guarantee.} We prove that DFL-AA removes the $(1-q_{ij})$ coefficient distortion in expectation, weighting each neighbor purely by AoI freshness~(Theorem~\ref{thm:unbiased}).

\item \textbf{Empirical validation.} We evaluate on EMNIST and CIFAR-10 across $20$- and $80$-node fixed directed topologies using discrete-event priority queue simulation, demonstrating consistent improvement over all baselines across loss levels $10\%$--$50\%$, with additional validation under heterogeneous per-link channel conditions and model deep model architectures. (Visit \url{https://github.com/DFedN/DFedNexus/tree/main} for reproducibility.) \\

\end{enumerate}

The rest of the paper is structured as follows. Section II reviews related work and discusses the gaps identified with positioning our work. Section III introduces the system model, including the network model, decentralized learning, and missing value reconstruction. Section IV presents the DFL-AA solution with all formulations. Section V demonstrates the validation of design choices. Section VI presents the performance evaluation, and Section VII concludes the paper and outlines future directions.

\section{Related Work}
\label{sec: related_work}

We review existing work across four dimensions relevant to DFL-AA and identify the gap that motivates our approach.

\begin{table}[t]
\caption{Comparison of DFL-AA with related methods.}
\label{tab:comparison}
\centering
\scriptsize
\setlength{\tabcolsep}{3pt}
\renewcommand{\arraystretch}{1.2}
\begin{threeparttable}
\begin{tabular}{lccccc}
\toprule
\textbf{Method} &
\textbf{Partial reception.} &
\textbf{Staleness} &
\textbf{Topology} &
\textbf{Decentral.} &
\textbf{Async} \\
\midrule
FedAvg~\cite{FedAvg}                        & Full only \xmark    & \xmark         & Star & \xmark & \xmark \\

FedBuff~\cite{fedBUFF}                      & Full only \xmark    & Buffer $\sim$  & Star & \xmark & \cmark \\
D-PSGD~\cite{lian_can_2017}                 & Full only \xmark    & \xmark         & Undir.      & \cmark & \xmark \\
SGP$^\dagger$~\cite{SGP}                              & Full only \xmark    & \xmark         & Dir.        & \cmark & \xmark \\
AD-PSGD~\cite{AD_PSGD}                      & Full only \xmark    & Bounded $\sim$ & Undir.      & \cmark & \cmark \\
SWIFT~\cite{SWIFT}                          & Full only \xmark    & Wait-free $\sim$  & Undir.     & \cmark & \cmark \\
Kwon et al.~\cite{LRA_SysNC}                & \cmark              & \xmark         & Star & \xmark & \xmark \\
AMA~\cite{partial_updates_veh}              & \cmark              & Round $\sim$         & Dynamic    & \cmark  & \cmark \\
Soft-DSGD~\cite{DFL_with_unreliable_comm}   & \cmark              & \xmark         & Undir.     & \cmark  & \xmark \\
\midrule
\textbf{DFL-AA (ours)}                      & \textbf{\cmark}     & \textbf{AoI \cmark}     & \textbf{Dir.}  & \textbf{\cmark}  & \textbf{\cmark }\\
\bottomrule
\end{tabular}
\begin{tablenotes}

\item[$\dagger$] Originally a distributed optimization algorithm; included as the canonical directed-graph gossip baseline.

\item \cmark~= fully addressed. \cmark~= fully addressed. $\sim$ = async-capable but handles staleness only via an assumption (bounded $\tau$) or a coarse mechanism (round counter, buffer), not via continuous decay weighting. \xmark~= not addressed. Undir. = undirected topology. Dir. = directed topology.
\end{tablenotes}
\end{threeparttable}
\end{table}

\subsection{Decentralized Federated Learning}
\label{sec: dfl_back}

To overcome the bottlenecks and single point of failure (SPoF) of client-server FL, decentralized federated learning (DFL) has emerged as a scalable and robust alternative \cite{dflsurver2024, CFLvsDFL}. Existing DFL research has explored heterogeneity-aware asynchronous learning that employed a coordinator that predicts the runtime configuration for the next round based on past data and probability distributions \cite{HADFL}, adaptive topology and communication optimization \cite{adaptive_configurationfor_heterogeneous_participants_in_DFL, dynamic-topology-optimization-for-efficient-and-DFL}, asynchronous model synchronization to mitigate stragglers \cite{AsyncFL_in_decentralized_topology_on_dynamic_avg_consensus}, and adaptive neighbor weighting strategies \cite{pFedgame_2024}. Semi-decentralized frameworks replace the single aggregator with cooperating edge servers that reach consensus via gossip \cite{low_latency_FL_mobile}. CoCo \cite{DFL_het_edge}, for instance, adapts both topology and compression using a consensus distance but relies on a central coordinator. While these approaches improve efficiency and scalability, they generally assume reliable communication and do not explicitly address partial model reception and update staleness in lossy wireless environments.

\subsection{Asynchronous and Staleness-Aware FL}
\label{sec:async}

Staleness naturally arises when nodes train and communicate at different rates. FedAsync and FedBuff address this in centralized settings by weighting or buffering delayed updates, but staleness is typically defined in round- or iteration-based terms (i.e., how many server/global steps have elapsed since an update was computed) rather than by wall-clock time \cite{fedAsync,fedBUFF}. AD-PSGD~\cite{AD_PSGD} allows fully decentralized asynchronous operation under a bounded-delay assumption. SWIFT~\cite{SWIFT} removes the bounded-delay requirement, achieving $O(1/\!\sqrt{T})$ convergence guarantees. FedASMU and BLADE~\cite{FedASMU, BLADE} follow similar principles, using dynamic, staleness-aware updates. However, all of these methods assume lossless delivery. They are designed for reliable networks and lack a mechanism to handle incomplete model receptions. However, in lossy wireless communication, model updates may be received only partially, resulting in incomplete information used during aggregation and, consequently, degrading learning performance.

\subsection{Partial Reception and Communication-Efficient FL}
\label{sec: partial}

A smaller body of work explicitly addresses the reception of partial models. Soft-DSGD~\cite{DFL_with_unreliable_comm} adopts a UDP-based fire-and-forget model and fills missing parameters with local values. AMA~\cite{partial_updates_veh} uses a similar local-fill strategy with link-quality thresholding. Kwon et al.~\cite{LRA_SysNC} reconstruct missing parameters via an SVD-based approximation combined with systematic network coding, but report 172\% packet overhead and operate in a centralized server setting that is not applicable to peer-to-peer topologies~\cite{LRA_SysNC}. Crucially, none of these methods corrects the \emph{selection bias} introduced by partial reception. Furthermore, none of these methods handle staleness as they assume synchronous or near-synchronous operation, which does not hold in genuinely asynchronous wireless deployments.

Communication-efficient approaches such as knowledge distillation~\cite{profecommunicationefficientdecentralizedfederated, 
prototype-based-DFL-for-heterogeneous-timevarying-IOT}, weight quantization with joint bandwidth allocation to minimize on-device energy under latency constraints \cite{energy_eff_FL_het_mobile_device}, adaptive sparse 
training~\cite{AEDFL} and RL-based clustering~\cite{large_scale_decentralized_async_edge_learning_with_device_het} reduce bandwidth consumption but assume reliable delivery and do not address the partial-reception regime.

\subsection{Wireless and IoT Deployments}
\label{sec: wireless_iot}

The partial reception problem is particularly acute in wireless deployments, where per-link channel quality varies with distance, interference, and node heterogeneity~\cite{survey_resource_constraint_IOT_FL}. Wireless FL research has explored vehicular networks~\cite{chen2025mobility_aware_multi_tasking, mob-fl, drl_based_FL_for_efficient_vehicular_caching_mgmnt}, UAV systems~\cite{a_FL_approach_to_QoS_forcasting_in_cellular_vehicular_comm}, and mobility-aware optimization~\cite{mobility_aware_dfl_with_joint_optimization_IMPORTANT}, but these works focus on scheduling and resource allocation under reliable communication rather than correcting aggregation bias from packet loss.

In IoT and sensor network settings, early work on in-network aggregation~\cite{TAG} established the importance of handling partial data collection in wireless sensor networks to reduce communication overhead. However, this work was designed to reduce aggregation overhead but does not explicitly consider partial model aggregation under partial reception. FL for IoT under communication constraints~\cite{park_wirelss, abad_2020} typically assumes reliable delivery or relies on centralized gateway aggregation, neither of which applies to our setting.

\subsection{Research Gap}
\label{sec:gap}

Table~\ref{tab:comparison} summarizes how DFL-AA relates to existing methods. No prior work jointly addresses partial reception and staleness in asynchronous wireless DFL: approaches that handle partial updates typically ignore staleness, while asynchronous methods assume lossless communication. In addition, round-based staleness metrics are poorly suited, where heterogeneous computation and communication delays create irregular round gaps that fail to reflect true model freshness; Age-of-Information (AoI) provides a more appropriate measure in such environments. In contrast, DFL-AA is the only method that jointly corrects selection bias due to lossy links and captures temporal staleness via AoI, operating on directed graphs without requiring synchronized rounds or lossless delivery assumptions, which are properties particularly valuable for IoT and mobile network deployments in a wireless setting where per-link channel quality is heterogeneous, and retransmission is infeasible.

\section{System Model}
\label{sec: system_model}

\subsection{Network Model}
\label{sec:network_model}

We consider $n$ nodes connected by a \emph{fixed directed graph} $\mathcal{G} = (\mathcal{V}, \mathcal{E})$ that remains fixed over time. A directed edge $(j,i) \in \mathcal{E}$ indicates that node $j$ can transmit to node $i$. Let $\mathcal{N}(i) = \{j : (j,i) \in \mathcal{E}\}$ denote the in-neighborhood of node $i$. In contrast, $\mathcal{N}_{out}(i)$ denote the out-neighborhood. We assume $\mathcal{G}$ is strongly connected.

\begin{assumption}[Strong connectivity]
\label{ass:conn}
The directed graph $\mathcal{G}$ is strongly connected, meaning that for every pair of nodes $(i,j)$, there exists a directed path from $i$ to $j$.
\end{assumption}

\begin{remark}
We restrict to fixed directed topologies, capturing scenarios where connectivity is preplanned rather than opportunistic. Extending to fully dynamic topologies is left for future work.

\end{remark}

\subsection{Learning Objective}
\label{sec:learning_obj}

Each node $i$ holds private dataset $\mathcal{D}_i$ drawn from local distribution $\mathcal{P}_i$. Each node maintains a local model $\mathbf{w}_i \in \mathbb{R}^d$ cooperating to minimize a global objective with $d$ model parameters:
\begin{equation}
    \min_{\mathbf{w} \in \mathbb{R}^d} F(\mathbf{w}) = \frac{1}{n} \sum_{i=1}^{n} f_i(\mathbf{w})
  \quad f_i(\mathbf{w}) = \mathbb{E}_{\xi \sim 
  \mathcal{D}_i}[\ell(\mathbf{w};\xi)].
\label{eq:fl_objective}
\end{equation}

where $f_i(\mathbf{w}) = \mathbb{E}_{\xi \sim \mathcal{D}_i}[\ell(\mathbf{w}; \xi)]$ is the local loss for node $i$ over its private data distribution $\mathcal{D}_i$, and $\ell(\cdot)$ is the loss function. We consider the non-IID setting $\mathcal{P}_i \neq \mathcal{P}_j$, modeled by Dirichlet-$\alpha$ label partitioning.

\subsection{Communication Model}
\label{sec:comm_model}

We adopt a fire-and-forget UDP-like communication model. Each node serializes its model $\mathbf{w}_i \in \mathbb{R}^d$ into $K$ fixed-size chunks, each containing $\lfloor d/K \rfloor$ parameters. Chunk $k$ on link $(j \to i)$ is received independently with probability $q_{ij} \in (0,1]$.

\begin{definition}[Bernoulli chunk loss]
Let $b_{ij}^{(k)} \sim \mathrm{Bernoulli}(q_{ij})$ independently for each chunk $k \in \{1,\ldots,K\}$. The parameter-level reception mask $\mathbf{m}_{ij} \in \{0,1\}^d$ is defined as:
\begin{equation}
  \mathbf{m}_{ij}[p] = b_{ij}^{(k(p))},
\end{equation}
where $k(p) \in \{1,\ldots,K\}$ denotes the chunk index to which parameter $p$ belongs.
\label{def:bernoulli}
\end{definition}

\begin{definition}[Completeness]
The observed completeness of a transmission from $j$ to $i$ is $c_{ij} = \frac{1}{K}\sum_{k=1}^{K} b_{ij}^{(k)}$, with $\mathbb{E}[c_{ij}] = q_{ij}$.
\label{def:completeness}
\end{definition}

\begin{assumption}[Homogeneous chunk loss]
\label{ass:homo_loss}
The Bernoulli loss probability $q_{ij} \in (0,1]$ is identical across all $K$ chunks on a given link $(j \to i)$: for all $k \in \{1,\ldots,K\}$. Every chunk on the same link experiences the same loss probability.
\end{assumption}

\noindent Local-fill reconstruction completes the received vector via element-wise multiplication:
\begin{equation}
  \hat{\mathbf{w}}_j = \mathbf{m}_{ij} \odot \tilde{\mathbf{w}}_j 
  + (\mathbf{1} - \mathbf{m}_{ij}) \odot \mathbf{w}_i,
  \label{eq:localfill}
\end{equation}
where $\odot$ denotes elementwise multiplication, $\mathbf{m}_{ij} \in \{0,1\}^d$ is the reception mask from Definition~\ref{def:bernoulli}, and $\tilde{\mathbf{w}}_j \in \mathbb{R}^d$ is the partially received model with arbitrary values at unobserved positions. Received parameters ($\mathbf{m}_{ij}[p]=1$) retain their transmitted values; missing parameters ($\mathbf{m}_{ij}[p]=0$) are substituted with the receiver's own values $\mathbf{w}_i[p]$.

\subsection{Asynchronous Operation}
\label{sec:async_model}

Nodes operate without a global synchronization barrier. Each node $i$ maintains:
\begin{itemize}[noitemsep]
\item Local model $\mathbf{w}_i \in \mathbb{R}^d$.
\item Inbox $\mathcal{B}_i$ storing the most recent $(\tilde{\mathbf{w}}_j, c_{ij}, t_{\mathrm{gen}}^j)$ from each in-neighbor, where $t_{\mathrm{gen}}^j$ is generation timestamp. Older entries are overwritten on the new reception.
\item EWMA estimate $\hat{q}_{ij}$ for each in-neighbor, updated on each reception.
\end{itemize}

\noindent\textbf{Age of Information.}
Each transmission carries a generation timestamp $t_{\mathrm{gen}}^j$ in the simulation's virtual time. Node $i$ computes the AoI reference as the most recent timestamp in its inbox, thereby no global clock is required:
\begin{equation}
  t_{\mathrm{ref}} = \max_{j \in \mathcal{B}_i} 
  t_{\mathrm{gen}}^j, \quad
  \mathrm{AoI}_{ij} = \max(0,\; 
  t_{\mathrm{ref}} - t_{\mathrm{gen}}^j).
  \label{eq:aoi}
\end{equation}

\section{The DFL-AA Algorithm}
\label{sec:algorithm}

\subsection{Design Rationale}

DFL-AA addresses two independent failure modes with two independent components:

\textbf{Spatial: IPW corrects selection bias.}
Under partial reception, low-quality links (low $q_{ij}$) contribute fewer messages, leading to systematic under-representation. This creates selection bias, where the received models are not a uniform sample of neighbors. DFL-AA uses the Horvitz–Thompson \cite{IPW} correction by weighting each received update with $1/\hat{q}_{ij}$, so weaker links are up-weighted to compensate for their lower visibility.

\textbf{Temporal: AoI decay discounts staleness.}
DFL-AA addresses the staleness issue by applying an exponential decay $\exp(-\mathrm{AoI}_{ij}/\tau)$, reducing the influence of outdated updates relative to fresh ones. This operates independently of IPW, where the first corrects spatial sampling bias, while the second accounts for temporal staleness.

Table~\ref{tab:notation} summarizes the notation used throughout Algorithm~\ref{alg:dflaa}, which presents DFL-AA at a single node $i$, and Figure~\ref{fig:architecture} illustrates the overall process of the proposed method. 

\begin{table}[h]
\caption{Notation used in Algorithm~\ref{alg:dflaa}}
\label{tab:notation}
\centering
\footnotesize
\setlength{\tabcolsep}{4pt}
\renewcommand{\arraystretch}{1.35}
\begin{tabular}{p{2.0cm} p{5.8cm}}
\toprule
\textbf{Symbol} & \textbf{Description} \\
\midrule
$\mathbf{w}_i^{(t)}$         & Local model at round $t$ \\
$\mathbf{w}_i^{(t+\frac{1}{2})}$ & Model after local training, before aggregation \\
$\mathcal{B}_i$              & Inbox: most recent transmission per in-neighbor \\
$\mathcal{B}_i^+$            & Inbox entries with $c_{ij}^{(t)} \geq c_{\min}$ \\
$\mathbf{w}_j^{(t)}$             & True complete model at node $j$ \\
$\tilde{\mathbf{w}}_j^{(t)}$ & Partially received model from neighbor $j$ \\
$\hat{\mathbf{w}}_j^{(t)}$   & Local-fill reconstructed model from $j$ \\
$\hat{q}_{ij}$               & EWMA estimate of reception rate on link $j \to i$ \\
$a_{ij}^{(t)}$               & Combined IPW $\times$ AoI weight for neighbor $j$ \\
\bottomrule
\end{tabular}
\end{table}

\subsection{EWMA Channel Estimation}
\label{sec:ewma_estimation}

Since the true reception rate $q_{ij}$ is unknown a priori and may vary with distance and interference, DFL-AA estimates it online via EWMA:
\begin{equation}
  \hat{q}_{ij} \leftarrow (1-\beta)\,\hat{q}_{ij} + \beta\,c_{ij},
  \label{eq:ewma}
\end{equation}
where $c_{ij}$ is the observed completeness of the most recent transmission from
$j$ and $\beta \in (0,1)$ is the EWMA forgetting factor. The estimate is
initialized to the first observed $c_{ij}$ and floored at $q_{\mathrm{floor}} > 0$ to prevent $1/\hat{q}_{ij}$ from becoming numerically degenerate:
$\hat{q}_{ij} \leftarrow \max(\hat{q}_{ij}, q_{\mathrm{floor}})$.

\subsection{Aggregation Rule}
\label{sec:aggr_rule}

The aggregation step at node $i$ is a normalized weighted average:
\begin{equation}
  \mathbf{w}_i^{\mathrm{new}} = \frac{\mathbf{w}_i + \sum_{j \in \mathcal{B}_i^+} a_{ij}\,\hat{\mathbf{w}}_j}
    {1 + \sum_{j \in \mathcal{B}_i^+} a_{ij}},
  \label{eq:aggregate}
\end{equation}
where $\mathcal{B}_i^+ = \{j \in \mathcal{B}_i : c_{ij} \geq c_{\min}\}$ excludes transmissions below the minimum completeness threshold $c_{\min}$ and $\hat{\mathbf{w}}_j$ is local-fill reconstructed model parameters according to Equation \eqref{eq:localfill} and $a_{ij}^{(t)} = (1/{\hat{q}_{ij}}) \cdot \exp (-{\mathrm{AoI}_{ij}^{(t)}}/{\tau})$ with $\tau$ as hyperparameter for decay.

\begin{remark}[Necessity of local-fill reconstruction]
Local-fill reconstruction~\eqref{eq:localfill} is a prerequisite for the IPW correction rather than an independent design choice. IPW reweights each neighbor's contribution by $1/\hat{q}_{ij}$ to correct for the underrepresentation of poor-quality links across rounds. However, this correction requires a complete, fixed-length parameter vector for aggregation. Without reconstruction, partially received models have variable support across coordinates, making a normalized weighted average undefined. Local-fill provides this fixed-length input at negligible cost: it requires no additional communication, no coordination with the sender, and only a single element-wise operation over $d$ parameters (Eq.~\eqref{eq:localfill}), which is dominated by the local SGD computation by several orders of magnitude and comparably low computation than other existing methods \cite{LRA_SysNC}, which uses SVD-based reconstruction from systematic network coding. Local-fill is therefore both necessary for IPW to operate and provably safe under the DFL-AA aggregation rule.
\end{remark}

\begin{figure*}[!t]
    \centering
    \includegraphics[width=\textwidth]{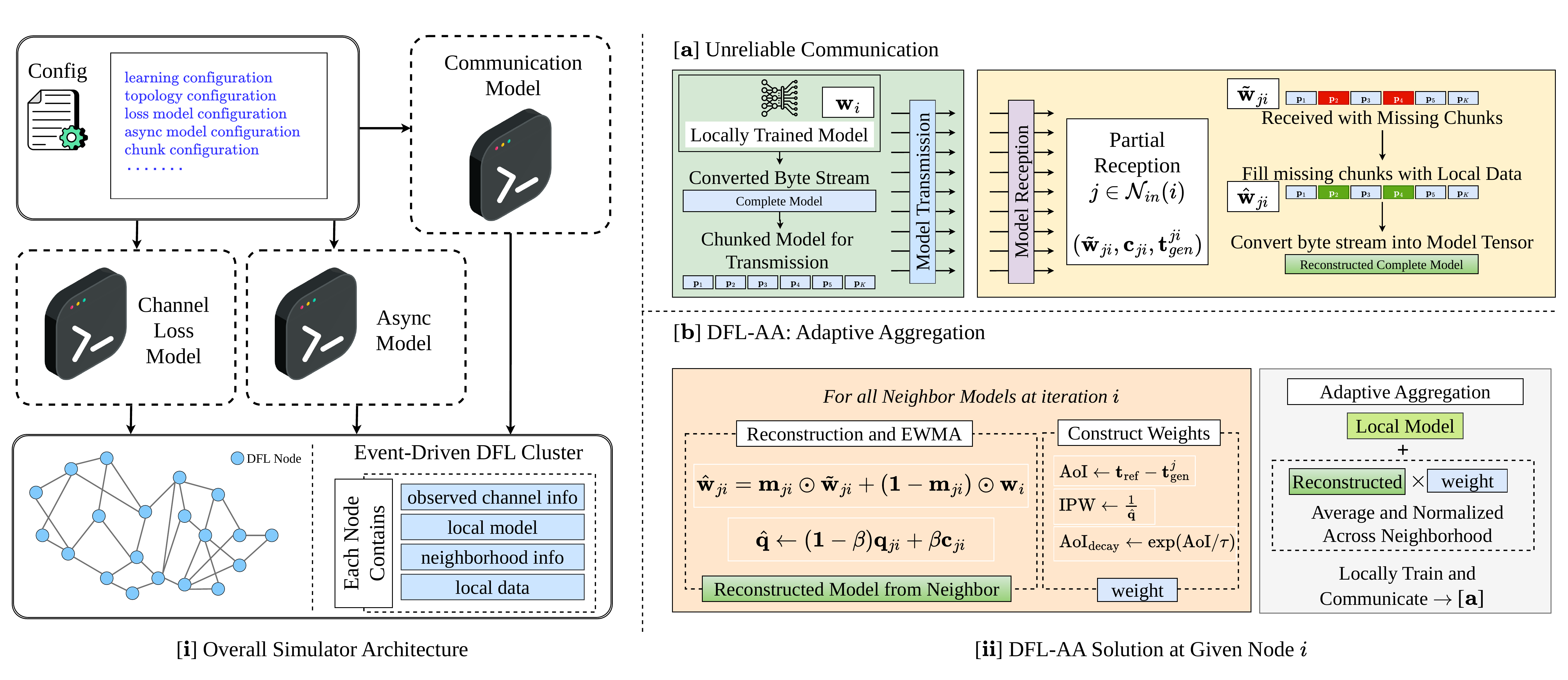}
    \caption{Overview of DFL-AA approach with simulation architecture, unreliable communication, and core DFL-AA architecture on a certain node.}
    \label{fig:architecture}
\end{figure*}

\begin{algorithm}[t]
\caption{\hlblue{DFL-AA} at node $i$}
\label{alg:dflaa}
\renewcommand{\algorithmicrequire}{\textbf{Input:}}
\renewcommand{\algorithmicensure}{\textbf{Output:}}
\begin{algorithmic}[1]
\REQUIRE Local data $\mathcal{D}_i$, initial model $\mathbf{w}_i^{(0)}$,
learning rate $\eta$, local steps $E$, AoI decay $\tau$, EWMA factor
$\beta$, minimum completeness $c_{\min}$, reception floor
$q_{\mathrm{floor}}$, time $T$
\ENSURE Final model $\mathbf{w}_i^{(T)}$
\STATE Initialize inbox $\mathcal{B}_i \leftarrow \emptyset$, and
$\hat{q}_{ij} \leftarrow 1.0$ for all $j \in \mathcal{N}(i)$
\FOR{$t = 0$ to $T-1$}
    \STATE \hlred{//~Phase 1: Local Training}
    \STATE $\mathbf{w}_i^{(t,0)} \leftarrow \mathbf{w}_i^{(t)}$
    \FOR{$e = 1$ to $E$}
        \STATE Sample mini-batch $\xi \sim \mathcal{D}_i$
        \STATE $\mathbf{w}_i^{(t,e)} \leftarrow \mathbf{w}_i^{(t,e-1)} -
        \eta \nabla f_i(\mathbf{w}_i^{(t,e-1)}; \xi)$
    \ENDFOR
    \STATE $\mathbf{w}_i^{(t+\frac{1}{2})} \leftarrow \mathbf{w}_i^{(t,E)}$
    \STATE $t_{\mathrm{gen}} \leftarrow t$
    \STATE \hlred{//~Phase 2: Broadcast}
    \STATE $\{(\mathbf{w}_i^{(t+\frac{1}{2}),(k)}, t_{\mathrm{gen}})\}_{k=1}^{K}
    \leftarrow \mathrm{Serialize}(\mathbf{w}_i^{(t+\frac{1}{2})})$
    \STATE Broadcast to all $j \in \mathcal{N}_{\mathrm{out}}(i)$
    (fire-and-forget)
    \STATE \hlred{//~Phase 3: Adaptive Aggregation}
    \STATE Receive partial models
    $\{(\tilde{\mathbf{w}}_j^{(t+\frac{1}{2})}, c_{ij}^{(t)},
    t_{\mathrm{gen}}^j)\}$ in $\mathcal{B}_i$
    \STATE $t_{\mathrm{ref}} \leftarrow
    \max_{j \in \mathcal{B}_i} t_{\mathrm{gen}}^j$
    \FORALL{$(\tilde{\mathbf{w}}_j^{(t+\frac{1}{2})}, c_{ij}^{(t)},
    t_{\mathrm{gen}}^j) \in \mathcal{B}_i$}
        \IF{$c_{ij}^{(t)} \ge c_{\min}$}
            \STATE $\hat{\mathbf{w}}_j^{(t+\frac{1}{2})} \leftarrow$
            Local-Fill$(\tilde{\mathbf{w}}_j^{(t+\frac{1}{2})})$ : Eq.~\eqref{eq:localfill}
            \STATE \hlblue{$\hat{q}_{ij} \leftarrow (1-\beta)\hat{q}_{ij} +
            \beta\, c_{ij}^{(t)}$}
            \STATE $\hat{q}_{ij} \leftarrow
            \max(\hat{q}_{ij},\, q_{\mathrm{floor}})$
            \STATE \hlblue{$\mathrm{AoI}_{ij}^{(t)} \leftarrow
            \max(0,\; t_{\mathrm{ref}} - t_{\mathrm{gen}}^j)$}
            \STATE \hlblue{$a_{ij}^{(t)} \leftarrow \dfrac{1}{\hat{q}_{ij}} \cdot
            \exp\!\left(-\dfrac{\mathrm{AoI}_{ij}^{(t)}}{\tau}\right)$}
        \ENDIF
    \ENDFOR
    \IF{$\sum_{j \in \mathcal{B}_i^+} a_{ij}^{(t)} > 0$}
        \STATE \hlblue{$\mathbf{w}_i^{(t+1)} \leftarrow
        \dfrac{\mathbf{w}_i^{(t+\frac{1}{2})} +
        \sum_{j \in \mathcal{B}_i^+} a_{ij}^{(t)}\,
        \hat{\mathbf{w}}_j^{(t+\frac{1}{2})}}
        {1 + \sum_{j \in \mathcal{B}_i^+} a_{ij}^{(t)}}$}
    \ENDIF
\ENDFOR
\STATE \textbf{output} $\mathbf{w}_i^{(T)}$
\end{algorithmic}
\end{algorithm}


\section{Design Validation}
\label{sec:design}

We prove that the two components of DFL-AA's weight formula~\eqref{eq:aggregate} are each necessary, and that natural alternatives are provably incorrect. Throughout, $q_{ij} \in (0,1]$ denotes the true per-chunk reception probability on link $(j \to i)$ from Definition~\ref{def:bernoulli}, and 
expectations are taken over chunk loss randomness.

\subsection{Failure Mode 1: Selection Bias}
\label{sec:failure_mode1}

\begin{lemma}[Expected local-fill reconstruction]
\label{lem:lf_expect}
Under Bernoulli chunk loss and local-fill reconstruction~\eqref{eq:localfill}, the expected reconstructed model at each parameter $p$ is:
\begin{equation}
  \mathbb{E}\!\left[\hat{\mathbf{w}}_j[p]\right]
  = q_{ij}\,{\mathbf{w}}_j[p] + (1-q_{ij})\,\mathbf{w}_i[p].
  \label{eq:lf_expect}
\end{equation}
\end{lemma}

\begin{IEEEproof}
From Definition~\ref{def:bernoulli}, parameter $p$ belongs to chunk $k(p)$, and $\mathbf{m}_{ij}[p] = b_{ij}^{(k(p))}$ where $b_{ij}^{(k(p))} \sim \mathrm{Bernoulli}(q_{ij})$.
From the local-fill equation~\eqref{eq:localfill}, $\hat{\mathbf{w}}_j[p]$ takes one of two values depending on whether chunk $k(p)$ was received:

\begin{equation}
\hat{\mathbf{w}}_{j}[p] =
\begin{cases}
\mathbf{w}_{j}[p], & \text{if } \mathbf{m}_{ij}[p] = 1, \\[1.0mm]
\mathbf{w}_{i}[p], & \text{if } \mathbf{m}_{ij}[p] = 0.
\end{cases}
\label{eq:local_fill_elementwise}
\end{equation}

Note that when $\mathbf{m}_{ij}[p] = 1$, the received value $\tilde{\mathbf{w}}_j[p]$ equals the true value $\mathbf{w}_{j}[p]$ since the chunk arrived intact. Taking expectation by conditioning on the two cases:

\begin{align}
\mathbb{E}\!\left[\hat{\mathbf{w}}_{j}[p]\right]
&=
\Pr\!\left(\mathbf{m}_{ij}[p]=1\right)\mathbf{w}_{j}[p]
+
\Pr\!\left(\mathbf{m}_{ij}[p]=0\right)\mathbf{w}_{i}[p]
\nonumber\\
&=
q_{ij}\,\mathbf{w}_{j}[p]
+
(1-q_{ij})\,\mathbf{w}_{i}[p].
\label{eq:local_fill_expectation}
\end{align}

where $\mathbf{w}_{j}[p]$ and $\mathbf{w}_{i}[p]$ are deterministic given the current model states and are therefore constant with respect to the Bernoulli reception process $\mathbf{m}_{ij}[p]$. We used $\mathbb{E}[\mathbf{m}_{ij}[p]] =
\mathbb{E}[b_{ij}^{(k(p))}] = q_{ij}$ (\textit{Assumption~\ref{ass:homo_loss}}).
\end{IEEEproof}

\begin{proposition}[Selection bias of uniform gossip]
\label{prop:lf_bias}
Under Bernoulli$(q_{ij})$ chunk loss, the expected 
uniform gossip aggregate:
\begin{equation}
  \mathbf{w}_i^{\mathrm{unif}} = 
  \frac{\mathbf{w}_i + \sum_{j \in \mathcal{N}(i)} 
  \hat{\mathbf{w}}_j}{1 + |\mathcal{N}(i)|}
\end{equation}
satisfies:
\begin{equation}
  \mathbb{E}\!\left[\mathbf{w}_i^{\mathrm{unif}}\right]
  = \mathbf{w}_i^{\mathrm{ideal}}
  - \underbrace{\frac{\sum_{j}(1-q_{ij})
    ({\mathbf{w}}_j - \mathbf{w}_i)}
    {1 + |\mathcal{N}(i)|}}_{\text{selection bias}},
  \label{eq:unif_bias}
\end{equation}
where $\mathbf{w}_i^{\mathrm{ideal}} = (\mathbf{w}_i + \sum_j {\mathbf{w}}_j)/(1+|\mathcal{N}(i)|)$ is the full-reception aggregate. The bias is irreducible under non-IID data and grows with both the loss rate $(1-q_{ij})$ and the model disagreement $\|\mathbf{w}_i - \mathbf{w}_j\|$.
\end{proposition}

\begin{IEEEproof}
Applying Lemma~\ref{lem:lf_expect} coordinate-wise and taking the expectation of the uniform aggregate:
\begin{align}
  \mathbb{E}\!\left[\mathbf{w}_i^{\mathrm{unif}}\right]
  &= \frac{\mathbb{E}[\mathbf{w}_i] + \sum_j \mathbb{E}
     [\hat{\mathbf{w}}_j]}{1+|\mathcal{N}(i)|} \notag\\
  &= \frac{\mathbf{w}_i + \sum_j 
     \left[q_{ij}{\mathbf{w}}_j + 
     (1-q_{ij})\mathbf{w}_i\right]}{1+|\mathcal{N}(i)|}.
  \label{eq:unif_expand}
\end{align}

\begin{align}
  &= \frac{\mathbf{w}_i + \sum_j {\mathbf{w}}_j + 
     \sum_j q_{ij}{\mathbf{w}}_j + \sum_j (1-q_{ij})\mathbf{w}_i - \sum_j {\mathbf{w}}_j}{1+|\mathcal{N}(i)|}  \notag\\
  &= \frac{\mathbf{w}_i + \sum_j {\mathbf{w}}_j}{1+|\mathcal{N}(i)|}
   - \frac{\sum_j(1-q_{ij}){\mathbf{w}}_j 
     - \sum_j(1-q_{ij})\mathbf{w}_i}{1+|\mathcal{N}(i)|} \notag\\
  &= \mathbf{w}_i^{\mathrm{ideal}}
   - \frac{\sum_j(1-q_{ij})
     ({\mathbf{w}}_j - \mathbf{w}_i)}{1+|\mathcal{N}(i)|}.
\end{align}
The bias vanishes only if $q_{ij} = 1$ for all $j$ (lossless links). Under partial reception with $q_{ij} < 1$, the expected aggregate under-weights each neighbor $j$. Under non-IID data, ${\mathbf{w}}_j \neq \mathbf{w}_i$ throughout training, so the bias is irreducible due to poor link quality.
\end{IEEEproof}

\begin{remark}[Bias as coefficient distortion]
\label{rem:bias_clarify}

The term $(\mathbf{w}_j-\mathbf{w}_i)$ in Proposition~\ref{prop:lf_bias} is the standard gossip update direction and is not itself the bias. The bias arises because local filling scales this direction by $q_{ij}$ rather than its ideal coefficient. As a result, low-quality links are systematically under-represented in the aggregate. DFL-AA corrects this distortion through IPW~\eqref{eq:unbiased_result}, making the numerator weight of each neighbor proportional to $\alpha_{ij}^*=\exp(-\mathrm{AoI}_{ij}/\tau)$, independent of $q_{ij}$.
\end{remark}

\subsection{Failure Mode 2: Push-Sum Weight Drain}
\label{sec:failure_mode2}

Push-sum~\cite{SGP} is the canonical method for gossip-based SGD on directed graphs and is the algorithmic family closest to DFL-AA's setting. However, we show that it is fundamentally incompatible with chunk-level packet loss, making it unsuitable for directed DFL under lossy communication.

\begin{proposition}[Push-sum weight drain]
\label{prop:drain}
In standard push-sum under Bernoulli$(q_{ij})$ chunk loss, where the weight variable $w_i$ is credited regardless of actual chunk reception, the total weight mass decays geometrically:
\begin{equation}
  \mathbb{E}\!\left[\sum_i w_i^{(t)}\right] 
  = \rho^t \cdot n, \quad \rho < 1.
\end{equation}
As $t \to \infty$, $w_i^{(t)} \to 0$ for all $i$, causing the model estimate $s_i/w_i$ to become numerically degenerate.
\end{proposition}

\begin{IEEEproof}
Summing the expected total 
weight held by all nodes after one round:
\begin{align}
  \mathbb{E}\!\left[\sum_i w_i^{(t+1)}\right]
  &= \sum_j w_j^{(t)} \cdot 
     \frac{1 + \sum_{i:(j,i)\in\mathcal{E}} q_{ij}}
     {1 + |\mathcal{N}_{\mathrm{out}}(j)|}
     \notag\\
  &\leq \sum_j w_j^{(t)} \cdot 
     \frac{1 + |\mathcal{N}_{\mathrm{out}}(j)|\,\bar{q}}
     {1 + |\mathcal{N}_{\mathrm{out}}(j)|},
  \label{eq:drain}
\end{align}
where $\bar{q} = \frac{1}{|\mathcal{E}|}
\sum_{(j,i)\in\mathcal{E}} q_{ij} < 1$ is the 
mean reception rate across all directed edges. 
The inequality holds because 
$\sum_{i:(j,i)\in\mathcal{E}} q_{ij} \leq 
|\mathcal{N}_{\mathrm{out}}(j)|\,\bar{q}$ by 
definition of $\bar{q}$. Let:
\begin{equation}
  \rho = \max_j 
  \frac{1 + |\mathcal{N}_{\mathrm{out}}(j)|\,\bar{q}}
  {1 + |\mathcal{N}_{\mathrm{out}}(j)|} < 1,
\end{equation}
since $\bar{q} < 1$ implies the numerator is 
strictly less than the denominator, we obtain:
\begin{equation}
  \mathbb{E}\!\left[\sum_i w_i^{(t+1)}\right] 
  \leq \rho\,\mathbb{E}\!\left[\sum_i w_i^{(t)}\right],
\end{equation}
which gives $\mathbb{E}[\sum_i w_i^{(t)}] \leq \rho^t \cdot n$ by induction. As $t \to \infty$, $w_i^{(t)} \to 0$ for all $i$, causing the model estimate $s_i/w_i$ to become numerically degenerate. 
\end{IEEEproof}




\noindent DFL-AA avoids push-sum entirely. The self-normalized denominator $1 + \sum_j a_{ij}^{(t)}$ in Eq.~\eqref{eq:aggregate} is recomputed each round using only locally available information, without maintaining auxiliary weight variables. This enables operation on directed graphs under lossy links, where push-sum-based methods such as SGP~\cite{SGP} fail.

\subsection{Why IPW Corrects Selection Bias}
\label{sec:dflaa_correct}

\begin{theorem}[Unbiasedness of DFL-AA aggregation]
\label{thm:unbiased}
Suppose $\hat{q}_{ij} = q_{ij}$ exactly for all 
$j \in \mathcal{N}(i)$. Then the expected DFL-AA 
aggregate satisfies:
\begin{equation}
  \mathbb{E}\!\left[\mathbf{w}_i^{\mathrm{new}}\right]
  = \mathbf{w}_i + 
  \frac{\sum_j \alpha_{ij}^*\,(\mathbf{w}_j - \mathbf{w}_i)}
  {1 + \sum_j \alpha_{ij}^*/q_{ij}},
  \label{eq:unbiased_result}
\end{equation}
where $\alpha_{ij}^* = \exp(-\mathrm{AoI}_{ij}/\tau)$. The $(1-q_{ij})$ coefficient distortion of Proposition~\ref{prop:lf_bias} is exactly eliminated: each neighbor's update direction $(\mathbf{w}_j - \mathbf{w}_i)$ is weighted purely by AoI freshness $\alpha_{ij}^*$, with no residual dependence on link quality $q_{ij}$ in the numerator.
\end{theorem}

\begin{remark}[$q_{ij}$ only affects normalization]
\label{rem:denom_qij}

Although $q_{ij}$ appears in the denominator $1+\sum_j \alpha_{ij}^*/q_{ij}$, it does not reintroduce the bias of Proposition~\ref{prop:lf_bias}. The bias originates from distorted coefficients on the update directions $(\mathbf{w}_j-\mathbf{w}_i)$. In Theorem~\ref{thm:unbiased}, these coefficients depend only on $\alpha_{ij}^*=\exp(-\mathrm{AoI}_{ij}/\tau)$ and are independent of $q_{ij}$. The denominator serves only as a normalization factor that adjusts the update magnitude.
\end{remark}

\begin{IEEEproof}
The DFL-AA aggregate~\eqref{eq:aggregate} with 
$\hat{q}_{ij} = q_{ij}$ is:
\begin{equation}
  \mathbf{w}_i^{\mathrm{new}} = 
  \frac{\mathbf{w}_i + \sum_j \frac{e^{-\mathrm{AoI}_{ij}/\tau}}
  {q_{ij}}\hat{\mathbf{w}}_j}
  {1 + \sum_j \frac{e^{-\mathrm{AoI}_{ij}/\tau}}{q_{ij}}}.
  \label{eq:dflaa_expand}
\end{equation}
We compute the expected numerator using Lemma~\ref{lem:lf_expect}:
\begin{align}
  &= \mathbb{E}\!\left[\mathbf{w}_i \right] + \sum_j \mathbb{E}\!\left[\frac{e^{-\mathrm{AoI}_{ij}/\tau}}
  {q_{ij}}\hat{\mathbf{w}}_j\right] \notag \\
  &= \mathbf{w}_i + \sum_j \frac{e^{-\mathrm{AoI}_{ij}/\tau}}{q_{ij}}
     \mathbb{E}\!\left[\hat{\mathbf{w}}_j\right] \notag\\
  &= \mathbf{w}_i + \sum_j \frac{e^{-\mathrm{AoI}_{ij}/\tau}}{q_{ij}}
     \left[q_{ij}\mathbf{w}_j + 
     (1-q_{ij})\mathbf{w}_i\right] \notag\\
  &= \mathbf{w}_i + \sum_j e^{-\mathrm{AoI}_{ij}/\tau}
   \mathbf{w}_j + \mathbf{w}_i \sum_j 
   \frac{(1-q_{ij})}{q_{ij}} e^{-\mathrm{AoI}_{ij}/\tau} \notag\\ 
  &= \mathbf{w}_i\!\left(1 + \sum_j 
     \frac{e^{-\mathrm{AoI}_{ij}/\tau}}{q_{ij}}\right)
   + \sum_j e^{-\mathrm{AoI}_{ij}/\tau}
     (\mathbf{w}_j - \mathbf{w}_i).
  \label{eq:ipw_term}
\end{align}

The expected denominator is:
\begin{align}
  \mathbb{E}\!\left[1 + \sum_j 
  \frac{e^{-\mathrm{AoI}_{ij}/\tau}}{q_{ij}}\right]
  &= 1 + \sum_j 
     \frac{e^{-\mathrm{AoI}_{ij}/\tau}}{q_{ij}}.
  \label{eq:denom_full}
\end{align}

\begin{equation}
  \mathbb{E}\!\left[\mathbf{w}_i^{\mathrm{new}}\right]
  = \mathbf{w}_i + 
  \frac{\sum_j e^{-\mathrm{AoI}_{ij}/\tau}
  (\mathbf{w}_j - \mathbf{w}_i)}
  {1 + \sum_j e^{-\mathrm{AoI}_{ij}/\tau}/q_{ij}},
\end{equation}
which simplifies to~\eqref{eq:unbiased_result} with $\alpha_{ij}^* = e^{-\mathrm{AoI}_{ij}/\tau}$ by collecting terms. Crucially, the $(1-q_{ij})/q_{ij}$ residual from the imputed values in~\eqref{eq:ipw_term} is absorbed into the $\mathbf{w}_i$ self-term and cancels exactly. This is the mechanism by which IPW eliminates selection bias.
\end{IEEEproof}

\begin{remark}[Finite EWMA residual]
\label{rem:ewma_residual}

When $\hat{q}_{ij} \neq q_{ij}$, the bias cancellation is approximate, leaving a residual of order $O(|\hat{q}_{ij} - q_{ij}| \cdot \|\mathbf{w}_j - \mathbf{w}_i\|)$. Both factors shrink during training as the EWMA converges geometrically to $q_{ij}$ and the models approach consensus, so the residual vanishes asymptotically.
\end{remark}

\subsection{Why Multiplicative Combination}
\label{sec:why_mul}

The IPW and AoI factors must be combined multiplicatively rather than additively. An additive combination $a_{ij} = 1/\hat{q}_{ij} + \exp(-\mathrm{AoI}_{ij}/\tau)$ would allow a very fresh but structurally under-represented neighbor to receive an inflated weight regardless of link quality, or vice versa. The multiplicative form enforces a \emph{both-must-be-good} criterion: a neighbor receives high weight only if it is both well-represented across rounds (low $1/\hat{q}_{ij}$ correction needed) and temporally fresh (low AoI).


\section{Performance Evaluation}
\label{sec:performance_eval}

\begin{table*}[t]
\centering
\caption{Performance across different data-loss levels in a 20-node asynchronous DFL setting. (Acc / Loss / AUC)}
\label{tab:main_acc_table}
\renewcommand{\arraystretch}{1.25}
\setlength{\tabcolsep}{5pt}

\begin{tabular}{llccc ccc ccc ccc ccc}
\toprule

& & 
\multicolumn{3}{c}{\textbf{FedAvg}} &
\multicolumn{3}{c}{\textbf{SWIFT}} &
\multicolumn{3}{c}{\textbf{AD-PSGD}} &
\multicolumn{3}{c}{\textbf{Soft-DSGD}} &
\multicolumn{3}{c}{\textbf{DFL-AA (Ours)}} \\
\cmidrule(lr){3-5} \cmidrule(lr){6-8} \cmidrule(lr){9-11} \cmidrule(lr){12-14} \cmidrule(lr){15-17}

\textbf{Dataset} & \textbf{Loss}
& Acc & Loss & AUC
& Acc & Loss & AUC
& Acc & Loss & AUC
& Acc & Loss & AUC
& Acc & Loss & AUC \\
\midrule

\multirow{5}{*}{EMNIST}

& 10\% 
& 29.74 & 7.96 & 27260
& 72.03 & 0.90 & 62289
& 70.20 & 0.95 & 58507
& 51.72 & 3.71 & 45203
& \textbf{74.86} & \textbf{0.81} & \textbf{66194} \\

& 20\% 
& 29.74 & 7.96 & 27260
& 71.17 & 0.93 & 61254
& 69.68 & 0.97 & 58199
& 51.44 & 3.72 & 44907
& \textbf{74.53} & \textbf{0.82} & \textbf{65632} \\

& 30\% 
& 29.74 & 7.96 & 27260
& 70.43 & 0.95 & 59994
& 69.14 & 0.99 & 57521
& 51.02 & 3.74 & 44363
& \textbf{74.14} & \textbf{0.84} & \textbf{64841} \\

& 50\% 
& 29.74 & 7.96 & 27260
& 67.28 & 1.06 & 56378
& 66.88 & 1.08 & 54784
& 50.37 & 3.76 & 43092
& \textbf{73.19} & \textbf{0.87} & \textbf{62868} \\

\midrule

\multirow{5}{*}{CIFAR-10}

& 10\%
& 26.93 & 5.74 & 104641
& 41.71 & 2.27 & 149772
& 38.94 & 2.60 & 144799
& 37.39 & 3.63 & 143724
& \textbf{50.47} & \textbf{1.62} & \textbf{185392} \\

& 20\%
& 26.24 & 6.00 & 104511
& 40.08 & 2.40 & 145664
& 37.99 & 2.65 & 144258
& 36.35 & 3.75 & 142892
& \textbf{50.35} & \textbf{1.59} & \textbf{183707} \\

& 30\%
& 26.44 & 5.97 & 104442
& 39.41 & 2.43 & 140693
& 39.03 & 2.46 & 142942
& 37.78 & 3.58 & 143168
& \textbf{50.24} & \textbf{1.58} & \textbf{182608} \\

& 50\%
& 26.71 & 5.94 & 104162
& 37.21 & 2.62 & 132485
& 39.54 & 2.46 & 136200
& 38.62 & 3.65 & 140288
& \textbf{49.39} & \textbf{1.56} & \textbf{178080} \\

\bottomrule
\end{tabular}
\end{table*}

\begin{figure*}[!t]
  \centering
  \includegraphics[width=\textwidth]{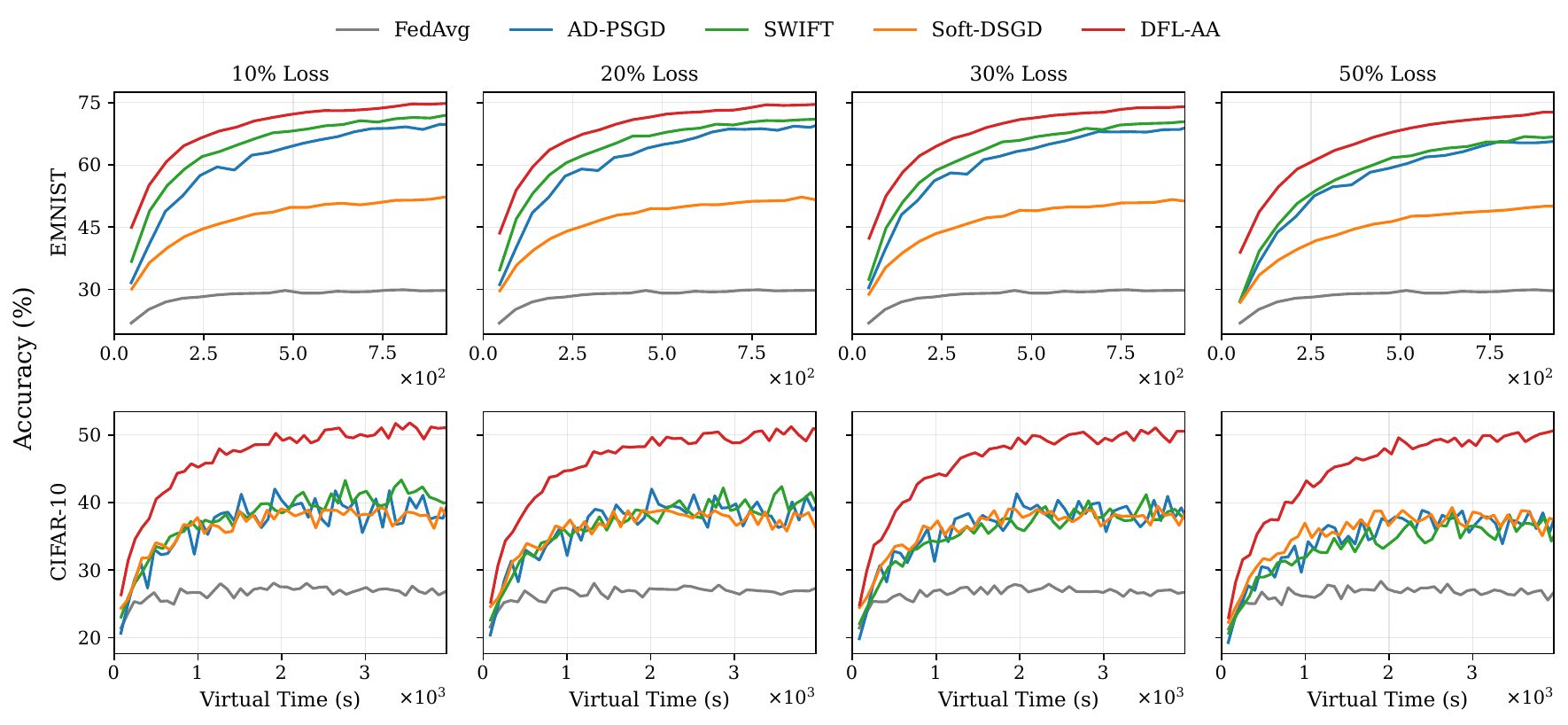}
  \caption{Test accuracy on EMNIST and CIFAR-10 datasets across lossy regimes (10--50\%) for 20-node setup with $\alpha=0.1$ heterogeneous data distribution for Erd\H{o}s-R'{e}nyi random graph with average degree 4.}
  \label{fig:accuracy_lossy}
\end{figure*}

\begin{figure*}[!t]
  \centering
  \includegraphics[width=\textwidth]{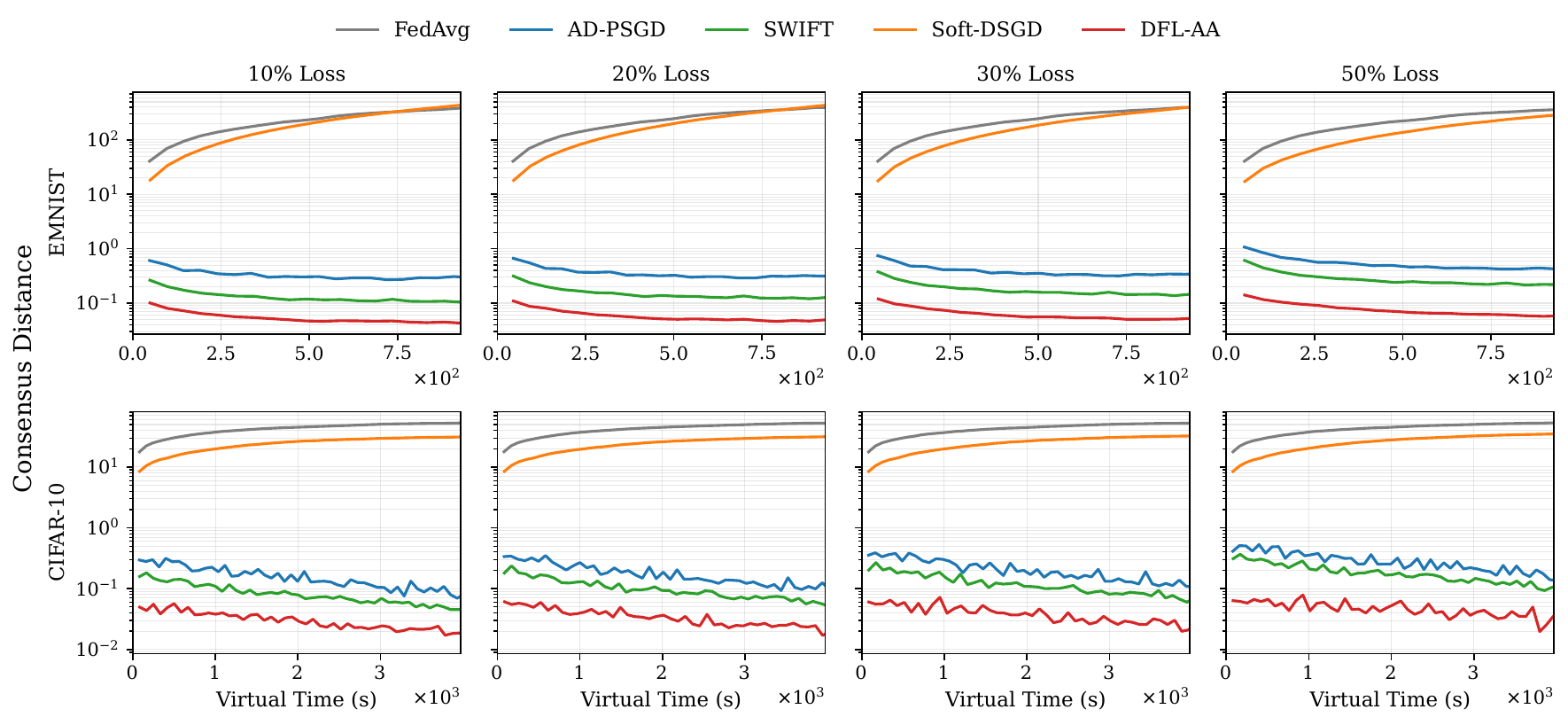}
  \caption{Consensus distance on EMNIST and CIFAR-10 datasets across lossy regimes (10--50\%) for 20-node setup with $\alpha=0.1$ heterogeneous data distribution for Erd\H{o}s-R'{e}nyi random graph with average degree 4.}
  \label{fig:consensus_lossy}
\end{figure*}

\begin{figure*}[!t]
  \centering
  \includegraphics[width=\textwidth]{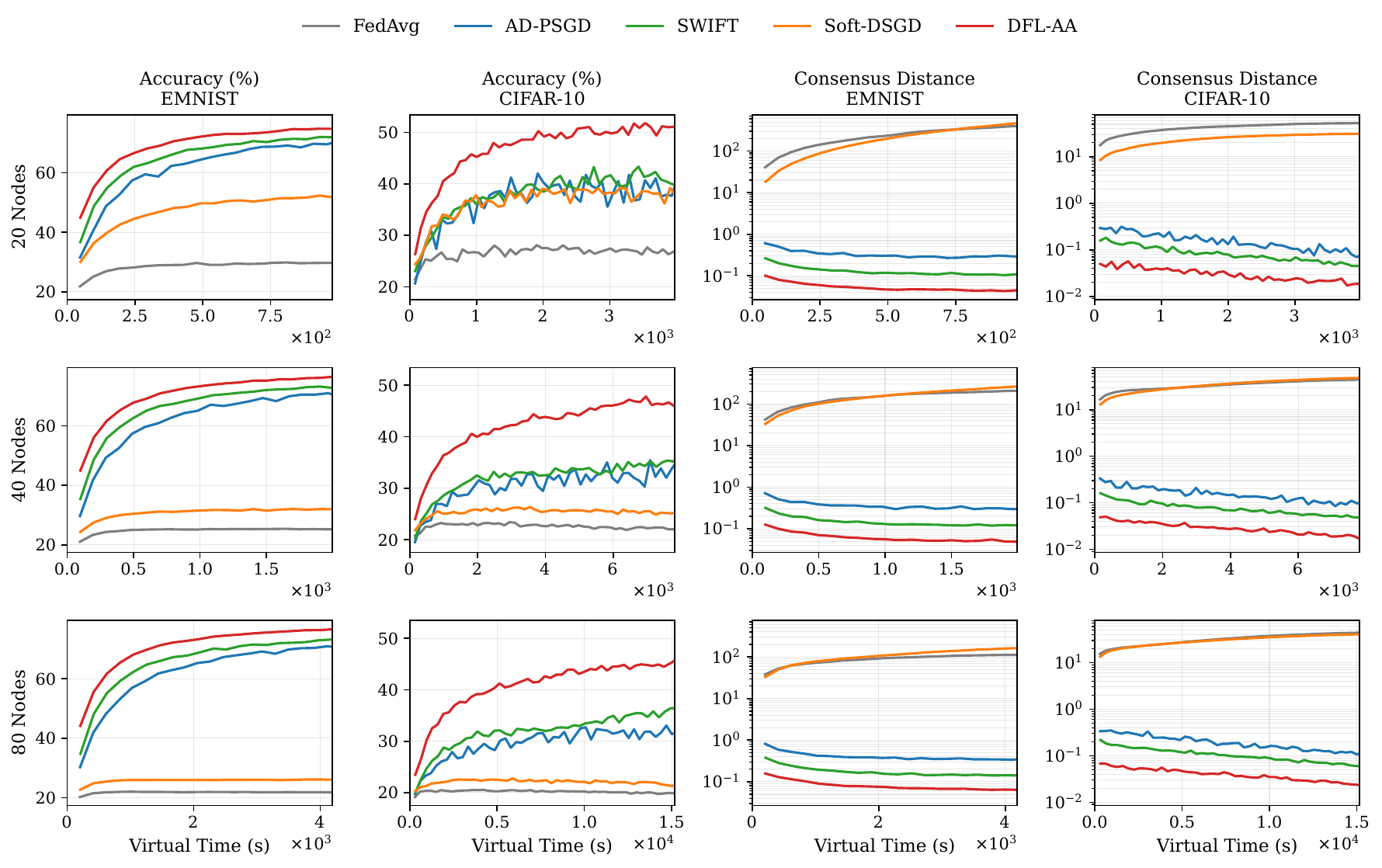}
  \caption{Scalability analysis of DFL-AA with each baseline over 10\% loss. Compare test accuracy and consensus distance on EMNIST and CIFAR-10 datasets with $\alpha=0.1$ heterogeneous data distribution for Erd\H{o}s-R'{e}nyi random graph with average degree 4.}
  \label{fig:node_scaling}
\end{figure*}

\begin{table*}[t]
\centering
\caption{Performance across different communication topologies in a 20-node asynchronous DFL setting under 10\% data loss. (Acc / Loss / AUC)}
\label{tab:topology_table}
\renewcommand{\arraystretch}{1.25}
\setlength{\tabcolsep}{5pt}

\begin{threeparttable}

\begin{tabular}{llccc ccc ccc ccc ccc}
\toprule

& &
\multicolumn{3}{c}{\textbf{FedAvg}} &
\multicolumn{3}{c}{\textbf{SWIFT}} &
\multicolumn{3}{c}{\textbf{AD-PSGD}} &
\multicolumn{3}{c}{\textbf{Soft-DSGD}} &
\multicolumn{3}{c}{\textbf{DFL-AA (Ours)}} \\
\cmidrule(lr){3-5} \cmidrule(lr){6-8} \cmidrule(lr){9-11} \cmidrule(lr){12-14} \cmidrule(lr){15-17}

\textbf{Dataset} & \textbf{Topology}
& Acc & Loss & AUC
& Acc & Loss & AUC
& Acc & Loss & AUC
& Acc & Loss & AUC
& Acc & Loss & AUC \\
\midrule

\multirow{3}{*}{EMNIST}
& (1)
& 29.74 & 7.96 & 27260
& 68.50 & 1.03 & \textbf{57866}
& \textbf{68.57} & \textbf{1.02} & 57682
& 29.79 & 7.98 & 27234
& 68.24 & 1.04 & 57863 \\

& (2)
& 29.74 & 7.96 & 27260
& 74.36 & 0.80 & 64725
& 72.64 & 0.86 & 60577
& 81.49 & 0.59 & 72811
& \textbf{81.59} & \textbf{0.58} & \textbf{73288} \\

& (3)
& 29.74 & 7.96 & 27260
& 72.03 & 0.90 & 62289
& 70.20 & 0.95 & 58507
& 51.72 & 3.71 & 45203
& \textbf{74.86} & \textbf{0.81} & \textbf{66194} \\

\midrule

\multirow{3}{*}{CIFAR-10}
& (1)
& 26.57 & 5.94 & 104257
& 41.40 & \textbf{2.27} & 144226
& 40.96 & \textbf{2.27} & 145076
& 27.08 & 5.63 & 104476
& \textbf{41.58} & 2.29 & \textbf{149041} \\

& (2)
& 26.71 & 5.99 & 104327
& 41.18 & 2.79 & 148995
& 37.68 & 2.31 & 142850
& 62.82 & 1.07 & 235581
& \textbf{63.59} & \textbf{1.04} & \textbf{234981} \\

& (3)
& 26.93 & 5.74 & 104641
& 41.71 & 2.27 & 149772
& 38.94 & 2.60 & 144799
& 37.39 & 3.63 & 143724
& \textbf{50.47} & \textbf{1.62} & \textbf{185392} \\

\bottomrule
\end{tabular}

\begin{tablenotes}[flushleft]
\footnotesize

\item
(1) Ring topology, where each node communicates with two neighbors in a circular structure.
(2) Fully connected topology, where each node communicates with all other nodes.
(3) Erd\H{o}s-R'{e}nyi random graph with average degree 4.

\end{tablenotes}
\end{threeparttable}
\end{table*}

\begin{table}[t]
\caption{Comparison of DFL-AA against EMNIST with different $\tau$ values under 10\% loss (mean AoI = 1.09, max AoI = 2.0)}
\label{tab:tau_ablation}
\centering
\scriptsize
\setlength{\tabcolsep}{3pt}
\renewcommand{\arraystretch}{1.2}
\begin{threeparttable}
\begin{tabular}{lcccc}
\toprule
\textbf{Method} &
\textbf{$\tau$ Value} &
\textbf{Accuracy(\%)} &
\textbf{Loss} &
\textbf{AUC} \\
\midrule
DFL-AA & 01 & 74.02 & 0.84 & 66229 \\
DFL-AA & 02 & 74.52 & 0.82 & 66265 \\
DFL-AA & 03 & 74.77 & 0.81 & 66252 \\
DFL-AA & 05 & 74.86 & 0.81 & 66194 \\
DFL-AA & 10 & 74.89 & 0.81 & 66090 \\
DFL-AA & 15 & 74.87 & 0.81 & 66104 \\
DFL-AA & 20 & 74.74 & 0.82 & 66114 \\
\bottomrule
\end{tabular}

\end{threeparttable}
\end{table}

\begin{table}[t]
\caption{Comparison of DFL-AA against EMNIST  with different staleness levels under 10\% loss and $\tau=5$}
\label{tab:staleness_table}
\centering
\scriptsize
\setlength{\tabcolsep}{3pt}
\renewcommand{\arraystretch}{1.2}
\begin{threeparttable}
\begin{tabular}{lccccc}
\toprule
\textbf{Mean AoI} &
\textbf{Max AoI} &
\textbf{Accuracy(\%)} &
\textbf{Loss} &
\textbf{Cons. Dist.} &
\textbf{AUC} \\
\midrule
1.09 & 2.0 & 74.86 & 0.81 & 0.05 & 66194 \\
1.38 & 3.5 & 73.79 & 0.84 & 0.05 & 63621 \\
1.96 & 5.0 & 73.05 & 0.87 & 0.06 & 62208 \\
\bottomrule
\end{tabular}

\end{threeparttable}
\end{table}

\begin{figure}[!t]
    \centering
    \includegraphics[width=\columnwidth]{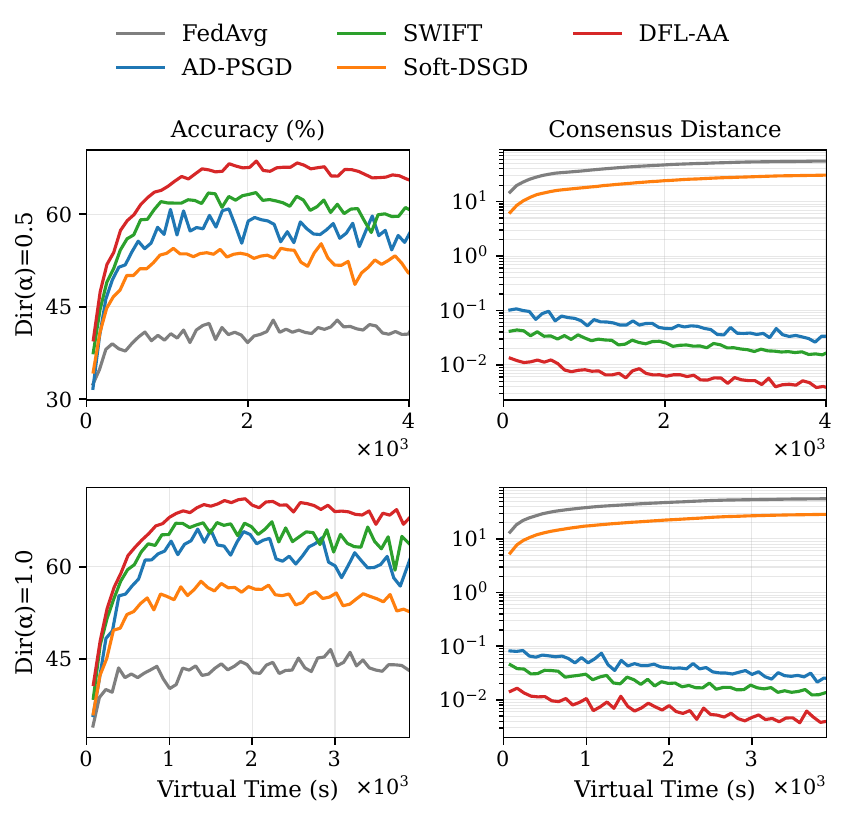}
    \caption{Evaluating test accuracy and consensus distance on the CIFAR-10 dataset with a 20-node setup across two heterogeneous data-distribution levels ($\alpha=0.5$ and $\alpha=1.0$) for Erd\H{o}s-R'{e}nyi random graph with average degree 4.}
    \label{fig:robustness_alphas}
\end{figure}

\begin{figure}[t]
    \centering
    \includegraphics[width=\columnwidth]{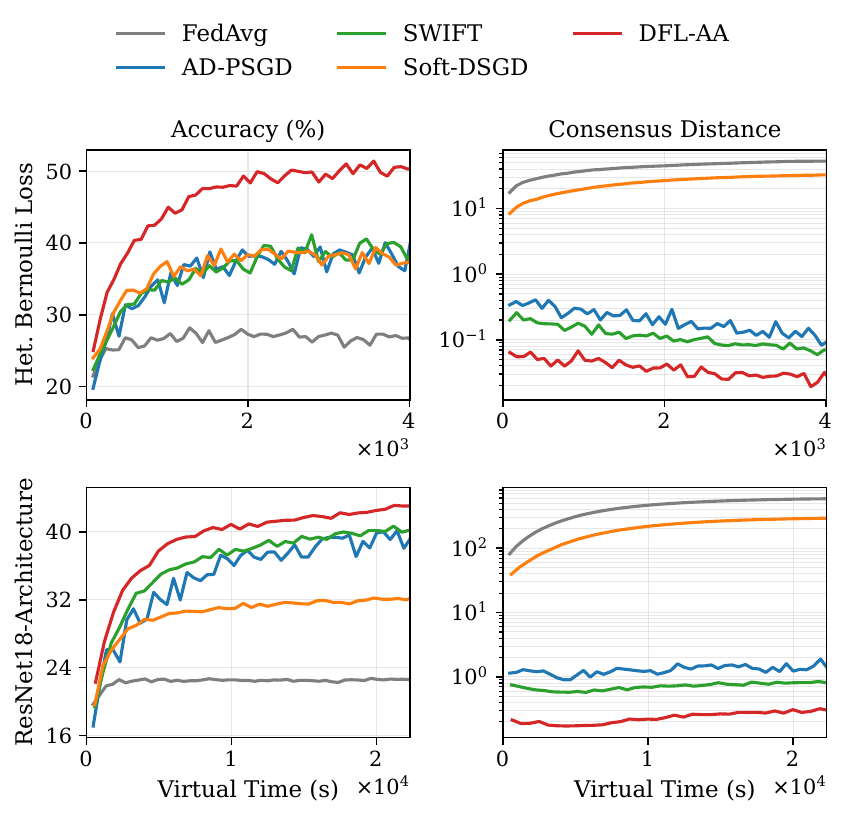}
    \caption{Heterogeneous Loss and Resnet}
    \label{fig:het_loss}
\end{figure}

\begin{figure}[!t]
    \centering
    \includegraphics[width=\columnwidth]{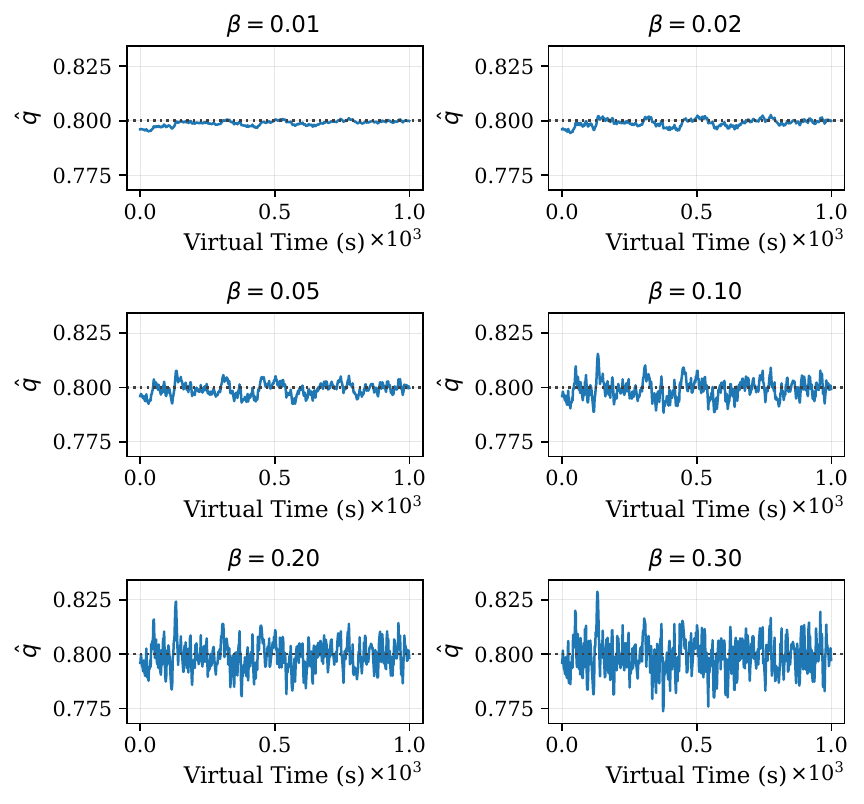}
    \caption{EWMA estimation of the true reception rate of $0.8$ (dashed) for six forgetting factors $\beta \in \{0.01, 0.02, 0.05, 0.10, 0.20, 0.30\}$ on a tracked link ($n=20$ nodes, EMNIST, $20\%$ loss).}
    \label{fig:ewma_converge}
\end{figure}

\paragraph{Simulator}
All experiments are conducted in a custom discrete-event simulator built on a priority-queue event engine. The system schedules local training completion, chunk transmission, inbox updates, and aggregation as timestamped events that are processed in causal order. We use virtual time as the x-axis for all training curves since the round number is not meaningful under asynchronous execution. Throughout experiments, we assume sufficient 
bandwidth and model packet loss as the primary impairment, consistent with prior work~\cite{DFL_with_unreliable_comm}. Bandwidth-aware chunk sizing is left for future work. All experiments are conducted on a virtual machine equipped with an L40s-16c GPU (24\,GB vRAM), 16 vCPUs, and 118\,GB RAM.

\paragraph{Datasets and Models}
We evaluate on two standard FL benchmarks using a non-IID Dirichlet ($ \alpha = 0.1$) partitioning scheme, which represents strong label heterogeneity for our main experiments.
\begin{itemize}
    \item \textbf{EMNIST}~\cite{EMNIST}:
    $28\times28$ grayscale handwritten characters with 47 classes. Model: two-layer MLP with hidden dimensions 256--256, LayerNorm after each layer, ReLU activations, and dropout~(0.1). Optimizer: Adam, learning rate $0.001$.

    \item \textbf{CIFAR-10}~\cite{CIFAR}: 
    $32\times32$ RGB images with 10 classes. Model: CNN with three stages of $3\times3$ standard convolutions (64--128--128 channels), BatchNorm, ReLU, a single residual skip at the widest stage, and a 256-unit FC head with dropout~(0.2). Optimizer: SGD with momentum $0.9$, learning rate $0.01$, weight decay $5\times10^{-4}$.
\end{itemize}

\paragraph{Training Configuration}

We set local epochs $E = 3$, batch size 64, EWMA factor $\beta = 0.05$, $c_{\min} = 0.1$, $q_{\mathrm{floor}} = 0.05$, and and packet size 1400~bytes following standard MTU constraints~\cite{MTU_ref}, which determines the number of chunks $K = \lceil d \cdot s_p / 1400 \rceil$ where $d$ is the model dimension and $s_p$ is the byte size per parameter, which is set as 4. The AoI decay constant $\tau$ is set to $5.0$, and the sensitivity analysis is provided in Table~\ref{tab:tau_ablation}.

\paragraph{Topology and Loss Levels}
We use random fixed directed Erd\H{o}s-R'{e}nyi graphs with average degree of 4. We evaluate on $n \in \{20, 40, 80\}$ nodes. Chunk loss is applied homogeneously across links with rates $\mathrm{loss} \in {10\%, 20\%, 30\%, 50\%}$. We also test the topology's robustness to ring and fully connected structures at a representative loss level and evaluate its robustness under heterogeneous per-link loss levels.

\paragraph{Baselines}
We compare against \textbf{FedAvg}~\cite{FedAvg} (decentralized variant which drops partial updates), \textbf{Soft-DSGD}~\cite{DFL_with_unreliable_comm} (local-fill reconstruction and gossip weighting), \textbf{AD-PSGD}~\cite{AD_PSGD} (asynchronous decentralized SGD assuming full delivery), and \textbf{SWIFT}~\cite{SWIFT} (wait-free asynchronous DFL under reliable delivery assumptions). All baselines are implemented in the simulator under identical topology, loss, and data partitioning settings. Each asynchronous method (AD-PSGD and SWIFT) is adapted with local-fill reconstruction for a fair comparison, since these methods were not originally designed to handle packet loss.

\subsection{Main Results}

Table~\ref{tab:main_acc_table} reports test accuracy, loss, and AUC for all methods with $n=20$ nodes across all loss levels under a heterogeneous data distribution among nodes (Dirichlet $\alpha=0.1$). DFL-AA achieves the highest accuracy in every configuration across both datasets. As loss increases, the advantage of DFL-AA grows monotonically, reaching \textbf{73.19\%} on EMNIST and \textbf{49.39\%} on CIFAR-10 at $50\%$ loss, outperforming the next-best method (SWIFT) by 5.91~pp and 12.18~pp, respectively. This monotonic widening directly validates Proposition~\ref{prop:lf_bias}, the $(1-q_{ij})$ coefficient distortion grows with loss rate, and DFL-AA's IPW correction becomes proportionally more valuable as the channel degrades. CIFAR-10 experiments in this setting achieve only the highest accuracy of 50.47\% due to high heterogeneity and limited information flow from the 4 average neighborhoods. Further experiments at low levels of data heterogeneity and with fully connected topologies demonstrate higher accuracy, as shown in Figure~\ref{fig:robustness_alphas} and Table~\ref{tab:topology_table}.

Among the baselines, AD-PSGD and SWIFT perform second-best despite assuming full delivery, because we augment them with local-fill reconstruction to ensure a fair comparison. Soft-DSGD, the method designed for partial reception, underperforms on both EMNIST and CIFAR-10 compared to other baselines, except FedAvg, thereby validating that local-fill reconstruction alone cannot be robust in asynchronous environments and that the directional distortion introduced by aggregation, as in Proposition~\ref{prop:lf_bias}. Strikingly, FedAvg, which simply drops partial updates, remains a lower bound because there is no complete model for the lossy-link setup we used in the experiments. Therefore, due to its architectural design, FedAvg discards all updates from neighbors and trains only the local model.

DFL-AA emerges as the best solution across different channel loss levels on every metric, including accuracy, loss, and AUC, thereby validating our unbiasedness Theorem~\ref{thm:unbiased} for DFL-AA aggregation over asynchronous lossy channels.

\subsection{Impact of Packet Loss Rate}

Figure~\ref{fig:accuracy_lossy} and ~\ref{fig:consensus_lossy} show accuracy and consensus distance (on a log scale) over virtual time at loss levels of 10--50\% for both datasets. Three observations are consistent across all configurations. 

First, DFL-AA converges faster and achieves higher accuracy than all baselines across all loss levels, and remains stable at higher loss rates. The convergence speed advantage is most pronounced under a high loss rate (at $50\%$ loss) on EMNIST; DFL-AA reaches 70\% accuracy, while all other baselines fail to reach that level at that loss level. FedAvg remains below 30\% throughout the runtime due to its local-only training as described in the main results section. Furthermore, the visible gap between all the baselines and DFL-AA widens on both datasets as channel loss increases $10\% \rightarrow 50\%$, again validating Proposition~\ref{prop:lf_bias}. The reduction in accuracy, along with an increase in channel loss, is minimal in DFL-AA compared to the best baseline. As loss increases from 10\% to 50\%, DFL-AA degrades by only 1.67~pp on EMNIST and 1.08~pp on CIFAR-10, while the next-best method SWIFT degrades by 4.75~pp and 4.5~pp, respectively. A relative degradation 2.84$\times$ and 4.17$\times$ larger than DFL-AA on EMNIST and CIFAR-10, respectively, confirming that IPW correction becomes proportionally more effective as channel quality worsens.

Second, Soft-DSGD exhibits a degradation pattern due to its inability to operate in asynchronous environments, despite being designed to handle partial model reception. This confirms that the local-fill reconstruction alone is insufficient in asynchronous wireless systems.

Third, Figure~\ref{fig:consensus_lossy} shows that DFL-AA achieves strictly lower consensus distance than all baselines across all loss levels and both datasets. This confirms that DFL-AA not only improves individual node accuracy but also drives the network toward tighter model agreement, a property that follows from eliminating the directional distortion in Proposition~\ref{prop:lf_bias}, which would otherwise push different nodes' aggregates in systematically different directions depending on their local link quality profiles.

\subsection{Scalability Analysis}

Figure~\ref{fig:node_scaling} shows accuracy and consensus distance for $n \in \{20, 40, 80\}$ 
nodes at fixed $10\%$ link loss. DFL-AA maintains its advantage across all node counts on both datasets. As $n$ increases, absolute accuracy decreases slightly across all methods due to sparser effective neighborhoods at a fixed average degree of 4 with fewer data samples, but DFL-AA's relative advantage is preserved. Consensus distance for DFL-AA remains consistently below all baselines across all scales, confirming that the IPW correction scales correctly with network size and that DFL-AA is robust to scalability.

Notably, FedAvg's accuracy collapses at 40 and 80 nodes on both datasets, because it limits its training on a local model as it drops partial receptions ($10\%$ loss), and due to increasing node counts, local data belonging to each node is limited and local datasets become more class-imbalanced and smaller, making local-only training increasingly ineffective. There is another notable performance degradation for Soft-DSGD as the node count increases, whereas AD-PSGD, SWIFT, and DFL-AA remain stable. This might be due to extreme heterogeneity and limited local data at each node, as well as directional bias introduced by local filling and averaging, which Soft-DSGD is designed to mitigate.

\subsection{Topology Robustness}

Table~\ref{tab:topology_table} reports performance across three topologies at $10\%$ link loss for ring, fully connected, and Erd\H{o}s-R'{e}nyi random directed graph with average degree~4. 

DFL-AA maintains competitive performance across all topologies. On the fully connected graph, Soft-DSGD achieves the second-highest accuracy on both datasets, only being defeated by DFL-AA. The gap is tiny, and it validates the Soft-DSGD design, as the authors demonstrate its robustness on a fully connected topology.

In the ring topology, all methods suffer significant accuracy degradation due to poor network connectivity, where each node has only two neighbors, thereby limiting information flow. DFL-AA maintains parity with SWIFT and AD-PSGD on this topology, confirming that the IPW correction is topology-agnostic, operating on per-link reception statistics available locally and not depending on global topology knowledge. Furthermore, our direct competitor, Soft-DSGD, achieves competitive performance on the fully connected topology, which its authors tested; we demonstrate its degradation under low information flow (random neighbors and ring topologies), while DFL-AA maintains stability across topologies, demonstrating its robustness.

\subsection{Heterogeneity Level Analysis}

Figure~\ref{fig:robustness_alphas} illustrates the performance of each method on the CIFAR-10 dataset under $10\%$ link loss at two levels of heterogeneity ($\alpha=0.5$ and $\alpha=1.0$). As $\alpha$ increases, the heterogeneity of the data distribution decreases, helping each node achieve higher accuracy than at lower $\alpha$ values in the same configuration. Figure~\ref{fig:robustness_alphas}
 clearly shows that the gap between DFL-AA and other baselines decreases with higher $\alpha$ values (lower data heterogeneity), but DFL-AA still manages to outperform other baselines by a clear margin, validating IPW-correctness under lossy channels along with staleness decay. However, the consensus distance between DFL-AA and the other best baselines (AD-PSGD and SWIFT) is not clearly visible at low levels of heterogeneity. These experimental results emphasize the importance of DFL-AA in wireless lossy network systems, where participating nodes mostly have heterogeneous data and operate asynchronously, and DFL-AA performs extremely well.

To validate DFL-AA beyond Assumption~\ref{ass:homo_loss}, 1st row of Figure~\ref{fig:het_loss} evaluates performance under heterogeneous per-link channel conditions on CIFAR-10 with $n=20$ nodes, where each link independently draws its loss rate from $\mathrm{Uniform}(0.05, 0.50)$. DFL-AA achieves 49.7\% final accuracy, only 0.77~pp below its best homogeneous result at $10\%$ loss (Table~\ref{tab:main_acc_table}), while outperforming the next-best baseline (SWIFT, 40.6\%) by 9.1~pp. This confirms that the online EWMA channel estimation adapts effectively to heterogeneous link quality: each node independently estimates its per-link $\hat{q}_{ij}$ and adjusts IPW weights accordingly, without knowledge of the global loss distribution.

To assess architectural generality, we repeat the CIFAR-10 experiment ($10\%$ loss, $n=20$, degree-4 random directed Erd\H{o}s-R'{e}nyi graph) with ResNet-18~\cite{resnet} (~11M parameters) in place of the CNN (~592K parameters). 2nd row of Figure~\ref{fig:het_loss} demonstrate the results. DFL-AA maintains its advantage over all baselines across both architectures, achieving 43.9\% with ResNet-18 and 50.47\% with the CNN, versus the next-best baseline (SWIFT) at 40.4\% and 41.71\%, respectively. The lower absolute accuracy of ResNet-18 is expected under this regime because its larger parameter count is serialized into proportionally more packets, increasing per-round partial reception under fixed loss, while its higher capacity is harder to train under extreme non-IID ($\alpha=0.1$) with limited per-node data. Crucially, the relative ranking of methods is preserved, confirming that the aggregation contribution is architecture-agnostic, with IPW and AoI corrections operating on flattened parameter vectors and being independent of model depth or parameter count.

\subsection{Hyperparameter Sensitivity}

Table~\ref{tab:tau_ablation} reports DFL-AA accuracy on EMNIST under 10\% loss for seven values of the AoI decay constant $\tau \in \{1, 2, 3, 5, 10, 15, 20\}$. Accuracy is stable across the range $\tau \in [3, 15]$, with the peak at $\tau=10$ (74.89\%). A very small $\tau=1$ slightly overpenalizes stale updates, reducing accuracy by 0.87~pp relative to the peak. A very large $\tau$ also degrades accuracy as the AoI weight approaches uniformity. The mean AoI in this experiment is 1.09~s with a maximum of 2.0~s, and the stable range is $\tau \in [3, 15]$, suggesting a practical tuning rule that sets $\tau$ to 3--10$\times$ the expected mean AoI in the deployment.

Table~\ref{tab:staleness_table} shows DFL-AA accuracy across three staleness levels (mean AoI: 1.09, 1.38, 1.96~s) with fixed $\tau=5$. Accuracy degrades gracefully as staleness increases. A 0.87~s increase in mean AoI reduces accuracy by only 1.81~pp, confirming that the AoI decay term effectively discounts stale updates and maintains robust performance under increased asynchrony.

Figure~\ref{fig:ewma_converge} shows $\hat{q}_{ij}$ over virtual time for a tracked link with a true reception rate of $q_{ij}=0.8$, using a shared y-axis to make the asymptotic variance directly comparable across panels. All six values of $\beta$ yield estimates centered on the true rate throughout training, confirming that the online channel-estimation assumption underlying Theorem~\ref{thm:unbiased} holds in practice. The figure reveals a clear bias-variance trade-off, where $\beta=0.01$ yields an almost flat estimate with minimal variance, whereas $\beta=0.30$ converges at the same rate but exhibits roughly $5\times$ wider oscillations around the true value. The default $\beta=0.05$ maintains a tight, stable estimate. The residual bias $O(|\hat{q}_{ij}-q_{ij}| \cdot \|\mathbf{w}_j-\mathbf{w}_i\|)$ from Remark~\ref{rem:ewma_residual} is therefore negligible throughout training and EWMA tracking experiments confirm that $\hat{q}_{ij}$ converges to the true reception rate across all tested forgetting factors, validating the online estimation assumption underlying Theorem~\ref{thm:unbiased} in practice.

\section{Conclusions and Future Work}
\label{sec: conclusion}

We presented DFL-AA, an asynchronous gossip aggregation algorithm for decentralized federated learning over fixed directed graphs, with chunk-level Bernoulli packet loss over fixed directed wireless links. DFL-AA addresses two key failure modes not handled by classical gossip methods: selection bias, where low-quality links are under-represented in the inbox, and update staleness, where asynchronous nodes contribute models with varying temporal freshness. The IPW correction with online EWMA channel estimation removes the $(1-q_{ij})$ coefficient distortion in expectation, while the AoI decay discounts stale updates without requiring a global clock. Experiments on EMNIST and CIFAR-10 over 20-, 40-, and 80-node directed topologies show consistent gains over all baselines across loss rates from $10\%$ to $50\%$, with larger improvements under higher loss and stronger heterogeneity. Additional experiments under heterogeneous per-link channel conditions, varying data heterogeneity levels, and ResNet-18 confirm that DFL-AA generalizes beyond the homogeneous evaluation setting and remains effective regardless of channel distribution, data 
heterogeneity, or model architecture. Finally, EWMA tracking experiments confirm that $\hat{q}_{ij}$ converges to the true reception rate under reasonable forgetting factors.

The proposed DFL-AA framework also opens several promising directions for future research. One direction is to extend the IPW mechanism to temporally correlated losses in communication links, enabling robust operation under burst losses. Another direction is the incorporation of bandwidth-aware communication models, where chunk sizing and IPW design can be jointly optimized under practical network constraints. Future work may also investigate parameter-level aggregation strategies that operate directly on partial model updates, potentially improving communication efficiency and further reducing estimation bias. Finally, extending DFL-AA to Byzantine-resilient and heterogeneous-model decentralized learning scenarios represents an important research direction.

\bibliographystyle{IEEEtran}
\bibliography{references}

\end{document}